\documentclass{article}
\usepackage[final]{neurips_2023}

\usepackage{subfigure}

\usepackage{enumitem}

\usepackage{amsmath,amssymb,amsthm}
\usepackage{bm}
\usepackage{graphicx,here,comment}
\usepackage{algorithmic}
\usepackage{algorithm}
\usepackage{xcolor}
\usepackage{wrapfig}

\usepackage{url} 
\usepackage[pagebackref]{hyperref}
\usepackage{xcolor}
\hypersetup{colorlinks,linkcolor={red!70!black},citecolor={blue!65!black},urlcolor={blue!70!black}}
\usepackage{color}

\newtheorem{theorem}{Theorem}

\newtheorem{lemma}[theorem]{Lemma}
\newtheorem{proposition}[theorem]{Proposition}
\newtheorem{assumption}{Assumption}
\theoremstyle{definition}
\newtheorem{definition}{Definition}
\newtheorem{remark}{Remark}

\newcommand{\R}{\mathbb{R}}
\newcommand{\N}{\mathbb{N}}

\newcommand{\mG}{\mathcal{G}}

\newcommand{\mA}{\mathcal{A}}

\newcommand{\mH}{\mathcal{H}}

\newcommand{\Ep}{\mathbb{E}}
\renewcommand{\Pr}{\mathbb{P}}

\renewcommand{\hat}{\widehat}
\renewcommand{\tilde}{\widetilde}

\newcommand{\argmin}{\operatornamewithlimits{argmin}}
\newcommand{\argmax}{\operatornamewithlimits{argmax}}

\newcommand{\mone}{\textbf{1}}

\def\bX{\mathbf{X}}

\def\br{\mathbf{r}}

\newcommand{\trace}{\mathrm{tr}}

\DeclareMathOperator{\Err}{E}

\usepackage{color}

\newcommand{\updated}[1]{#1}

\usepackage{newtxtext,newtxmath}

\newcommand{\algoref}[1]{Algorithm~\ref{#1}}

\newcommand{\thmref}[1]{Theorem~\ref{#1}}
\newcommand{\lemref}[1]{Lemma~\ref{#1}}

\newcommand{\Ich}{I(t)}

\newcommand{\Xit}{X^{(i)}(t)}

\newcommand{\Xarmt}[1]{X^{(#1)}(t)}
\newcommand{\Xarmnt}[2]{X^{(#1)}_{#2}(t)}
\newcommand{\Xt}{X^{(I(t))}(t)}
\newcommand{\Yt}{Y^{(I(t))}(t)}
\newcommand{\Zit}{Z^{(i)}(t)}
\newcommand{\Xin}[1]{X^{(i)}_{#1}}

\newcommand{\Yin}[1]{Y^{(i)}_{#1}}

\newcommand{\bfXi}{\mathbf{X}^{(i)}}
\newcommand{\bfZi}{\mathbf{Z}^{(i)}}
\newcommand{\bA}{\mathbf{A}}

\newcommand{\bfYi}{\mathbf{Y}^{(i)}}

\newcommand{\thetai}{\theta^{(i)}}
\newcommand{\hattheta}{\hat{\theta}}
\newcommand{\hatthetai}{\hattheta^{(i)}}
\newcommand{\hatthetaarm}[1]{\hattheta^{(#1)}}
\newcommand{\thetaarm}[1]{\theta^{(#1)}}
\newcommand{\thetaarmn}[2]{\theta^{(#1)}_{#2}}
\newcommand{\Reg}{R}
\newcommand{\Ni}{N}

\newcommand{\Sigmai}{{\Sigma^{(i)}}}
\newcommand{\Sigmaarm}[1]{{\Sigma^{(#1)}}}
\newcommand{\Lambdai}{\Lambda^{(i)}}
\newcommand{\Vi}{V^{(i)}}
\newcommand{\lambdaij}{\lambda^{(i)}_j}
\newcommand{\lambdaik}{\lambda^{(i)}_k}
\newcommand{\lambdainum}[1]{\lambda^{(i)}_{#1}}

\newcommand{\uik}{u^{(i)}_k}

\newcommand{\biasiNT}{B_{N,T}^{(i)}}
\newcommand{\variNT}{V_{N,T}^{(i)}}

\newcommand{\nn}{\nonumber\\}
\newcommand{\ThetaScale}{\theta_{\mathrm{max}}}

\newcommand{\Clambda}{C_{\lambda}}
\newcommand{\hatSigmai}{{\hat{\Sigma}^{(i)}}}
\newcommand{\Stop}{\mathrm{Stop}}
\newcommand{\Stopi}{\mathrm{Stop}^{(i)}}

\newcommand{\poly}{\mathrm{poly}}
\newcommand{\CUpper}{C_U}
\newcommand{\cUpper}{c_U}
\newcommand{\CLower}{C_L}
\newcommand{\cLower}{c_L}
\newcommand{\CPoly}{{C_{\mathrm{poly}}}}
\newcommand{\CStopTr}{{C_{T}}}
\newcommand{\myeps}{{\epsilon_T}}

\usepackage[normalem]{ulem}
\DeclareRobustCommand{\erase}{\bgroup\markoverwith{\textcolor{red}{\rule[.5ex]{2pt}{0.4pt}}}\ULon}

\title{High-dimensional Contextual Bandit Problem without Sparsity}

\author{
  Junpei Komiyama \\
  New York University\\
  \texttt{junpei@komiyama.info} \\
  \And
  Masaaki Imaizumi \\
  The University of Tokyo / RIKEN Center for AIP\\
  \texttt{imaizumi@g.ecc.u-tokyo.ac.jp} \\
}

\begin{document}

\maketitle

\allowdisplaybreaks

\begin{abstract}
In this research, we investigate the high-dimensional linear contextual bandit problem where the number of features $p$ is greater than the budget $T$, or it may even be infinite. Differing from the majority of previous works in this field, we do not impose sparsity on the regression coefficients. Instead, we rely on recent findings on overparameterized models, which enables us to analyze the performance of the minimum-norm interpolating estimator when data distributions have small effective ranks. We propose an explore-then-commit (EtC) algorithm to address this problem and examine its performance. Through our analysis, we derive the optimal rate of the ETC algorithm in terms of $T$ and show that this rate can be achieved by balancing exploration and exploitation. Moreover, we introduce an adaptive explore-then-commit (AEtC) algorithm that adaptively finds the optimal balance. We assess the performance of the proposed algorithms through a series of simulations. 
\end{abstract}

\section{Introduction}

The multi-armed bandit problem \citep{robbins1952,Lairobbins1985} has been widely studied in the field of sequential decision-making problems in uncertain environments, and it can be applied to a variety of real-world scenarios. This problem involves an agent selecting one of $K$ arms in each round and receiving a corresponding reward. 
The agent aims to maximize the cumulative reward over rounds by using a clever algorithm that balances exploration and exploitation. 
In particular, a version of this problem called the contextual bandit problem \citep{abe1999,li2010} has attracted significant attention in the machine learning community. 
By observing the contexts associated with the arms, the agent can choose the best arm as a function of the contexts. This extension enables us to model many personalized machine learning scenarios, such as recommendation systems \citep{li2010,yuyan2022} and online advertising \citep{tang13}, and personalized treatments \citep{Chakraborty2014}.

Most of the papers about stochastic linear bandits assume that the number of features $p$ is moderate \citep{li2010,chu11a,abbasi2011improved}. When $p = o(\sqrt{T})$ to the number of rounds $T$, the model is identifiable, and the agent can choose the best arm for most rounds. However, recent machine learning models desire to utilize an even larger number of features, and the theory of bandit models under the identifiability assumption does not necessarily reflect the modern use of machine learning. 
Several recent papers have overcome this limitation by considering sparse linear bandit models \citep{wang2018minimax,kim2019doubly,BastaniB20,hao2020high,oh2021sparsity,li2022simple,JangZJ22}. Sparse linear bandit models accept a very large number of features\footnote{Typically, the number of feature $p$ can be exponential to the number of datapoints $T$.} and suppress most of the coefficients by introducing the $\ell 1$ regularize. 

That said, the sparsity imposed by such models limits the applicability of these models.
For example, in the case of recommendation models based on factorization, each user is associated with a dense latent vector \citep{Rendle10,Agarwal2012PersonalizedCS,yuyan2022}, which implies the sparsity is not unlikely the case. 
Another possible drawback of sparse models is that it requires the condition number to be close to one (e.g., the restricted isometry property, see \citet{van2009conditions} for review). 
This implies that the quality of the estimator is compromised by the noise on the features that correspond to small eigenvalues. 
Furthermore, it is still non-trivial to select a proper value of the penalty coefficient as a hyper-parameter. 
For example, \citet{hara2017enumerate} claims that small changes in the choice of coefficients significantly alter feature selection, and \citet{miolane2021distribution} show a limitation of the conventional theory on the choice of penalty coefficients.

In this paper, We consider an alternative high-dimensional linear bandit problem without sparsity. We allow $p$ to be as large as desired, and in fact, we even allow $p$ to be infinitely large.
Such an overparameterized model has more parameters than the number of training data points. 
A natural estimator in such a case is an \textit{interpolating} estimator, which perfectly fits the training data. 
We adopt recent results that bound the error of the estimator in terms of the \textit{effective rank} \citep{koltchinskii2017concentration,bartlett2020benign} on the covariance of the features. 
When the eigenvalues of the covariance decay moderately fast, we can obtain a concentration inequality on the squared error of the estimator.

The contributions of this paper are as follows: First, We consider explore-then-commit (EtC) strategy for the stochastic bandit problem based on the minimum-norm interpolating estimator. We derive the optimal rate of exploration that minimizes regret. However, EtC requires model-dependent parameters on the covariance, which limits the practical utility. To address this limitation, we propose an adaptive explore-then-commit (AEtC) strategy, which adaptively estimates these parameters and achieves the optimal rate. We conduct simulations to verify the efficiency of the proposed method.

\section{Preliminary}

\subsection{Notation}
For $z \in \N$, $[z] := \{1,2,\dots,z\}$.
For $x \in \R$, the notation $\lfloor x \rfloor$ here denotes the largest integer that is less than or equal to a scalar $x$. 
For vectors $X,X' \in \R^p$, $\langle X,X'\rangle := X^\top X'$ is an inner product, $\|X\|_2^2 := \langle X,X \rangle$ is an $\ell 2$-norm.
For a positive-definite matrix $A \in \R^p \times \R^p$, $\|X\|_\bA^2 := \langle X, \bA X \rangle$ is a weighted $\ell 2$-norm.
$\|\bA\|_{\mathrm{op}}$ denotes an operator norm of $\bA$. 
$O(\cdot), o(\cdot), \Omega(\cdot), \omega(\cdot)$ and $\Theta(\cdot)$ denotes Landau's Big-O, little-o, Big-Omega, little-omega, and Big-Theta notations, respectively.
$\tilde{O}(\cdot),\tilde{\Omega}(\cdot)$, and $\tilde{\Theta}(\cdot)$ are the notations that ignore polylogarithmic factors.

\subsection{Problem Setup}

This paper considers a linear contextual bandit problem with $K$ arms.
We consider the fully stochastic setting, where the contexts, as well as the rewards, are drawn from fixed distributions.
For each round $t \in [T]$ and arm $i \in [K]$, we define $\Xit$ as a $p$-dimensional zero-mean sub-Gaussian vector.
We assume $\Xit$ is independent among rounds (i.e., vectors in two different rounds $t,t'$ are independent) but allow vectors $\Xarmt{1},\Xarmt{2},\dots,\Xarmt{K}$ to be correlated with each other.
The forecaster chooses an arm $\Ich \in [K]$ based on the $\Xit$ values of all the arms, and then observes a reward that follows a linear model as shown in\begin{align}
\Yt = \langle \Xt, \thetaarm{\Ich} \rangle + \xi(t). \label{eq:linear_model}
\end{align}
The unknown true parameters $\thetai$ for each arm $i \in [K]$ lie in a parameter space $\Theta \subset \R^p$, and the independent noise term $\xi(t)$ has zero mean and variance $\sigma^2 > 0$. We assume that $\xi(t)$ is sub-Gaussian, and for the sake of simplicity, we assume that it does not depend on the choice of the arm. However, our results can be extended to the case where $\xi(t)$ varies among arms.
We assume that each $\thetai$ are bounded $\|\thetai\|_2 \leq \ThetaScale$.
For each $i \in [K]$, we define a covariance matrix $\Sigmai = \Ep[\Xarmt{i}\Xarmt{i}^\top] \in \R^{p \times p}$.

We define $i^*(t) := \argmax_{i \in [K]} \langle \Xit, \thetai \rangle $ as the (ex ante) optimal arm at round $t$.
Our goal is to design an algorithm that maximizes the total reward, which is equivalent to minimizing the following expected regret \citep{Lairobbins1985,AuerCF02};
\begin{align}
    \Reg(T) := \Ep \left[ \sum_{t=1}^T  \left(\langle \Xarmt{i^*(t)}, \thetaarm{i^*(t)} \rangle - \langle \Xt, \thetaarm{\Ich} \rangle\right) \right], \label{def:regret}
\end{align}
where the expectation is taken with respect to the randomness of the contexts and (possibly) on the choice of arm $\Ich$.

\subsection{\updated{Theory of Overparametrized Models}}

\updated{
The primary focus of this paper is on scenarios where the number of arms $K$ is moderate, but the number of features $p$ is greater than the budget $T$, possibly to the point of being infinite.
The efficiency of linear regression models in such scenarios depends significantly on the covariance matrix, $\Sigma^{(i)}$.
Unlike sparsity, the theory of benign overfitting tightly examines errors using the decay of the eigenvalues of the covariance matrix of the context.
In particular, if the decay rate of the eigenvalues is at a certain level, the error in linear regression converges to zero, even in high-dimensional spaces.
To provide a more rigorous analysis of eigenvalue decay, the concept of \textit{effective rank} is introduced in Section \ref{subsec_ranks}.
}

\begin{remark}{\rm (Dependence on $T$)}
In accordance with \citet{bartlett2020benign}, we permit the covariance matrix $\Sigmai$ to depend on $T$. 
In other words, we consider the sequence of covariances $\Sigmai(1),\Sigmai(2),\dots$ for each $T=1,2,\dots$. 
\updated{
The linear regression is consistent if the effective rank of $\Sigmai(T)$ grows sufficiently slow as a function of $T$.
}
\end{remark}

\subsection{Effective Ranks of Covariance Matrix}\label{subsec_ranks}

For a covariance matrix $\Sigmai$ for $i \in [K]$ and $k \in [p]$, let $\lambdaik$ be its $k$-th largest eigenvalue, such that $\Sigmai = \sum_{k=1}^p \lambdaik \uik (\uik)^\top$ with order $\lambdainum{1} \geq \lambdainum{2} \geq \cdots \geq \lambdainum{p}$ and eigenvectors $\uik$.
We define the concept of \textit{effective rank} as
\begin{align}
    r_k(\Sigmai) := \frac{\sum_{j > k} \lambdaij}{\lambdainum{k+1}}, ~~\mbox{and}~~R_k(\Sigmai) := \frac{(\Sigma_{j > k} \lambdaij)^2}{\Sigma_{j > k} (\lambdaij)^2}. \label{def:effective_rank}
\end{align}
The first quantity $r_k(\Sigmai)$ is related to the trace of $\Sigmai$, and the second quantity $R_k(\Sigmai)$ is related to the decay rate of the eigenvalues.

\begin{figure*}[htbp]
\includegraphics[width=\hsize]{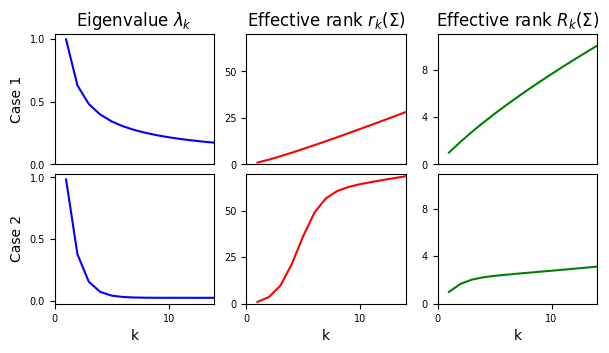}
    \caption{Eigenvalues $\lambda_k$ (left), the effective rank $r_k(\Sigma)$ (middle), and the effective rank $R_k(\Sigma)$ (right) with $k = 1,...,15$ and $T=3$; Case 1: $\lambda_k = C k^{-1 + 1/T^{0.99}}$, and Case 2: $\lambda_k =C( \exp(-k) + T \exp(-T) / p)$.
    \updated{Case 1 considers a decay that is slightly faster than $k^{-1}$, whereas Case 2 considers an exponentially fast decay. The slower increase of $r_k(\Sigma)$ and $R_k(\Sigma)$ in Case 2 reflects the impact of the eigenvalues' faster decay.}
    } 
    \label{fig:ranks}
\end{figure*}
\updated{
The effective rank is used as a measure of the intrinsic complexity of high-dimensional data, by rigorously capturing a decay rate of eigenvalues.
If the covariance matrix $\Sigmai$ is an identity matrix of size $p$, then $r_0(\Sigmai)$ and $R_0(\Sigmai)$ are both equal to $p$ (the rank of $\Sigmai$). However, we anticipate that these quantities will be less than $p$, which enables learning with fewer samples.
In Figure \ref{fig:ranks}, we plot of the effective rank in $k$ with certain cases of eigenvalues $\{\lambda_k\}_{k}$ of $\Sigma$. 
Even though $\Sigma$ is full-rank, the effective rank $R_k(\Sigma)$ is sublinear in $k$, which reflects the intrinsic low-complexity of data.
}

\updated{
Because of this property, the effective ranks are used in modern high-dimensional statistics. 
\cite{koltchinskii2017concentration} evaluated concentration inequalities for high-dimensional random matrices using the notion. 
\cite{bartlett2020benign,tsigler2020benign} evaluated the prediction error of a high-dimensional linear regression using the effective ranks and showed the consistency of the prediction.
}

\section{Explore-then-Commit with Interpolation}

The \textit{Explore-then-Commit} (EtC) algorithm is a well-known approach for solving high-dimensional linear bandit problems, and it has been shown to be effective in previous studies such as \citet{hao2020high,li2022simple}. The EtC algorithm operates by first conducting $T_0 = NK < T$ rounds of exploration, during which it uniformly explores all available arms to construct an estimator $\hatthetai$ for each arm $i \in [K]$. After the exploration phase, the algorithm proceeds with exploitation.
Let $N$ be the number of the draws of arm $i$, and $\bfXi = (\Xin{1},...,\Xin{N})^\top \in \R^{N \times p}$ and $\bfYi = (\Yin{1},...,\Yin{N})^\top \in \R^{N}$ be the observed contexts and rewards of arm $i$, where $(\Xin{n},\Yin{n})$ is the corresponding values on the $n$-th draw of arm $i$.
Since we choose $\Ich$ uniformly during the exploration phase, these are independent and identically drawn samples from the corresponding distribution. 

For estimating the parameter $\thetai$, we consider the minimum-norm \textit{interpolating estimator} that perfectly fits the data, which reveals its advantage in recent results on high-dimensions \citep{bartlett2020benign}. 
Rigorously, we assume $p > N$ and and consider the following definition:
\begin{align*}%
    \hatthetai := \argmin\left\{\|\theta\|_2 \,\bigg|\, \theta \in \R^p, 0 = \sum_{(\Yin{n},\Xin{n}): n \le N} (\Yin{n} - \langle \theta, \Xin{n} \rangle)^2 \right\}.
\end{align*}
This estimator has a simple representation
$
    \hatthetai= (\bfXi)^\top (\bfXi (\bfXi)^\top )^{-1} \bfYi.
$
The EtC algorithm is presented in Algorithm \ref{alg:esc}, which utilizes the aforementioned interpolating estimator. 
In the subsequent sections, we will first discuss the assumptions on the data-generating process, followed by an analysis of the accuracy of the estimator. We then present an upper bounds on the regret of the EtC algorithm.

\begin{algorithm}
\caption{Explore-then-Commit (EtC)}\label{alg:esc} 
\begin{algorithmic}
\REQUIRE Exploration duration $T_0$.
\FOR{$t=1,.., T_0$}
    \STATE Observe $\Xit$ for all $i \in [K]$.
    \STATE Choose $\Ich = t - K\lfloor t/K \rfloor$.
    \STATE Receive a reward $\Yt$.
\ENDFOR
\FOR{$i \in [K]$}
    \STATE $\hatthetai \leftarrow (\bfXi)^\top (\bfXi (\bfXi)^\top )^{-1} \bfYi$.
\ENDFOR
\FOR{$t = T_0 + 1,...,T$}
    \STATE Observe $\Xit$ for all $i \in [K]$.
    \STATE Choose $\Ich = \argmax_{i \in [K]}  \langle \Xit, \hatthetai \rangle$.
    \STATE Receive a reward $\Yt$.
\ENDFOR
\end{algorithmic}
\end{algorithm}

\section{Theory of Explore-then-Commit}

\updated{
This section analyzes the EtC algorithm.
If we were to fix the number of samples, this theory would largely align with the existing literature \citep{bartlett2020benign,tsigler2020benign}. However, the unique challenge in our case arises from the fact that the EtC selects the number of samples used to construct an estimator to balance between exploration and exploitation.
}

\subsection{Assumption and Notion}

We consider a spectral decomposition $\Sigmai = \Vi \Lambdai (\Vi)^\top$ with the matrices (operators) $\Vi$ and $\Lambdai$ for each arm $i\in[K]$, then impose the following assumptions. 
\begin{assumption}[sub-Gaussianity]\label{asmp:subgaussian}
For all $t \in \N$ and $i \in [K]$, the followings hold:
\begin{itemize}
\item There exists a random vector $\Zit$ such that $\Xit = \Vi (\Lambdai)^{1/2} \Zit$ which is independent elements and sub-Gaussian with a parameter $\kappa_x^2 > 0$, that is, for all $\lambda \in \R^p$, we have
            $\Ep[\exp( \langle \lambda , \Zit \rangle)] \leq \exp(\kappa_x^2 \|\lambda\|_2^2 / 2)$.
\item     Moreover, $\xi(t)$ is conditionally sub-Gaussian with a parameter $\kappa_\xi^2 > 0$, that is, for all $\lambda \in \R$, we have
            $\Ep[\exp(\lambda \xi(t)) \mid \Xit] \leq \exp(\kappa_\xi^2 \lambda^2/2)$.
\end{itemize}

\end{assumption}

\updated{The first bullet point of Assumption \ref{asmp:subgaussian} requires that the features are multivariate sub-Gaussian and their tail is not too heavy, which is important for concentration inequalities on random variables \citep{vershynin2018high}.
The second bullet point of Assumption \ref{asmp:subgaussian},  which states that rewards involve sub-Gaussian noise, is a common assumption in contextual bandit problems.
} 

\subsection{Benign Covariance}\label{subsec_benigncov}

We impose a condition to be benign on the covariance matrix $\Sigmai$ for all $i \in [K]$.
The condition is described using the notion of effective/coherent rank, which is commonly applied to the study of benign overfitting \citep{bartlett2020benign,tsigler2020benign}. However, unlike those papers that estimate using all $T$ samples, we only use a portion of the data for learning.

We first define the \textit{coherent rank} of $\Sigmai$ with a number of samples $N$ as
\begin{align*}
    k^*_N =  k^*_N(\Sigmai,N) := \min \{k \in \N \cup \{0\} \mid r_k(\Sigmai) \geq N\},
\end{align*}
where we define $\min \{\emptyset\} = +\infty$.
\updated{By using the coherent rank, we decompose the error into two quantities} called \textit{effective bias/variance} denoted as ${\biasiNT}$ and ${\variNT}$, based on a budget of $T$ and the number of samples $N$ used for estimation.
\begin{align}
    \biasiNT := \lambdainum{k^*_N}\,, \mbox{~and~} \variNT := \left( \frac{ k^*_N }{N} +  \frac{N}{R_{k^*_N}(\Sigmai)} \right). \label{def:effective_bias_var}
\end{align}

\begin{definition}[Benign covariance] \label{def:benign_covariance}
\updated{Under Assumption \ref{asmp:subgaussian},} a covariance matrix $\Sigmai$ is \textit{benign}, if $\biasiNT = o(1)$ and $\variNT = o(1)$ hold 
    as $N,T \to \infty$ while $N = \Theta (T)$. 
\end{definition}%
To put it differently, \updated{Assumption \ref{asmp:subgaussian} and Definition \ref{def:benign_covariance} state that,} if we can achieve consistent estimation by using all the data for estimation, then the covariance matrix is considered benign.

Technically, the benign property implies that eigenvalues decay fast enough compared with $T$. In particular, the following examples have been considered in the literature.
\begin{proposition}[Example of Benign Covariance, Theorem 31 in \citet{bartlett2020benign}\footnote{It should be noted that \cite{bartlett2020benign} only provided the variance term. Later on, \cite{tsigler2020benign} described both the variance and bias terms, which we follow.}
]
\label{prop_examples}
    A covariance matrix $\Sigmai$ with eigenvalues $\lambdainum{1},\lambdainum{2},...$ is benign if it satisfies one of the following.
    \begin{itemize}[leftmargin=*]
        \item \textbf{Example 1:} Let $a \in (0,1)$, $p = \infty$ and $\lambdainum{k} = k^{-(1+1/T^a)}$. In this case, we have $\biasiNT = O((T^{a}/N)^{1+1/T^a})$ and $\variNT = O(T^a/N + T^{-a})$. For this model to be learnable, $N \gg T^a$ is required, and the variance dominates the bias. 
        \item\textbf{Example 2:} Let $\lambdainum{k} = k^{-b}$ and $p = p_T = O(T^{c})$ with $b \in (0,1)$ and $c \in \left(\max\left(1, \frac{2}{2-b}, \frac{1}{1 - b^2}\right), \frac{1}{1-b}\right)$.
        In this case, we have $\biasiNT = O(T^{c(1-b)}/N)$ and $\variNT = O((N/T^{c})^{\max(1, \frac{1-b}{b})})$. For this model to be learnable, $N \gg T^{c(1-b)}$ is required, and the bias dominates the variance.
     \end{itemize}
\end{proposition}
The first example is when the decay rate is near $k^{-1}$, but the trace is $O(T^a)$. The second example is when the decay rate is smaller than $k^{-1}$, and the trace is $O(T^{c(1-b)})$. 
Note that \cite{bartlett2020benign} provided two other examples of the benign covariance matrices.\footnote{One of the omitted examples is the case where eigenvalues decay slightly slower than Example 1. The other example is the case where eigenvalues decay exponentially but has a noise term.}
Our analysis mainly focuses on the two examples in Proposition \ref{prop_examples}, but another example is also empirically tested in Section \ref{sec_sim}.

\subsection{Estimation Error by Exploration}

We evaluate an error in the estimator $\hatthetai$ by its prediction quality.
That is, with a covariance matrix $\Sigmai$ and an identical copy $\Xin{*}$ of $\Xin{1}$, we study
\begin{align*} 
    \|\thetai - \hatthetai\|_{\Sigmai}^2 = \Ep_{\Xin{*}} \left[(\langle {\thetai}, \Xin{*}\rangle - \langle \hatthetai, \Xin{*}\rangle)^2\right],
\end{align*}

The following result bounds the error of the estimator $\hatthetai$ in terms of bias and variance, which is a slight extension of \citet{tsigler2020benign}.
For $k \in [p]$, we define an empirical submatrix as $\mathbf{X}^{(i)}_{\updated{k:\infty}} \in \R^{N\times (p-k) }$ as the $p-k$ columns to the right of $\mathbf{X}^{(i)}$, and define a Gram sub-matrix $A_k^{(i)} = \mathbf{X}_{\updated{k:\infty}}^{(i)}(\mathbf{X}_{\updated{k:\infty}}^{(i)})^\top \in \R^{N \times N}$.
\begin{theorem}\label{thm_ub}
    Suppose Assumption \ref{asmp:subgaussian}.
    If there exists $\cUpper > 1$ such that $k^*_N < N/\cUpper $ and a condition number of $A_k^{(i)}$ is positive with probability at least $1 - \cUpper  e^{-N/\cUpper}$, then we have 
    \begin{equation} 
         \|\hatthetai - \thetai\|_\Sigmai^2 \leq \CUpper \left( \biasiNT + \variNT \right) ,
        \label{ineq_upper}
    \end{equation}
    with probability at least $1 - 2 \cUpper e^{-N/\cUpper}$.
\end{theorem}
Theorem \ref{thm_ub} implies that the estimation error converges to zero if $\Sigmai$ has the benign property. 
\updated{The assumption on the condition number of $A_k^{(i)}$ has a sufficient condition, which is provided in Lemma 3 in \cite{tsigler2020benign}.}

Moreover, the following lemma implies the tightness of the analysis in Theorem \ref{thm_ub}.
\begin{lemma}[Lower bound of estimation error, Theorem 10 in \citet{tsigler2020benign}]\label{lem_lower}
Suppose Assumption \ref{asmp:subgaussian}.
These exist some constants $\cLower, \CLower > 0$ such that, with probability at least $1 - \cLower e^{-N/\cLower}$, we have
\begin{equation}
    \|\hatthetai - \thetai\|_\Sigmai^2 
    \ge \CLower \left(\biasiNT + \variNT\right).
\end{equation}
\end{lemma}
In other words, the upper bound in Theorem \ref{thm_ub} is optimal up to a constant.\footnote{Note that the constant here can depend on model parameters.}

\subsection{Regret Bound of Explore-then-Commit}

This section analyzes the EtC algorithm. 
We introduce an error function $\Err(N, T)$ that characterizes the rate of regret, which can be obtained by considering $(\biasiNT + {\variNT})^{1/2}$.
\begin{assumption}{\rm (Error function\footnote{The coherent rank $k^*_N$ as well as $\biasiNT, \variNT$ are discrete in $N,T$, and the error function $\Err(N, T)$ is introduced to circumvent the issues related to this discreteness.})}
\label{asmp_error}
There exists a continuous function $\Err: \mathbb{R}^2 \to \mathbb{R}^+$ that is decreasing in $N$ such that $\|\hatthetai - \thetai\|_\Sigmai = \tilde{\Theta}\left( \Err(N, T) \right)$ as $N,T \to \infty$.
\end{assumption}
Assumption \ref{asmp_error} is satisfied in Examples 1 and 2. In particular, for Example 1 in Proposition \ref{prop_examples}, the error function is given as $\Err(N, T) = \sqrt{T^{a}/N + T^{-a}}$, while for Example 2 in Proposition \ref{prop_examples}, it is given as $\Err(N, T) = \sqrt{T^{c(1-b)}/N}$.

\begin{theorem} \label{thm:upper_bound}
Suppose Assumptions \ref{asmp:subgaussian}--\ref{asmp_error}.
    Suppose that we run the EtC algorithm (\algoref{alg:esc}).
    Then, regret \eqref{def:regret} satisfies 
    \begin{align*}
    R(T)  = \tilde{O}(\updated{L(T, K)}),
\end{align*}
as $T \to \infty$ with some $\alpha > 0$.
\updated{where $L(T, K)$ is such that}
\begin{equation}\label{amount_exploration}
\updated{
N^* = \min_{N} \{N: N \ge T \Err(N/K, T)\}.
}
\end{equation}
\end{theorem}
\updated{
Specifically, for Example 1 in Proposition \ref{prop_examples}, we have $L(T,K) = \max(T^{(2+a)/3}K^{2/3}, T^{1-a/2})$, whereas for Example 2 in Proposition \ref{prop_examples}, we have $L(T, K) = T^{(2 + c(1-b))/3} K^{2/3}$.
}
\begin{proof}[Proof sketch of Theorem \ref{thm:upper_bound}]
We show that the regret-per-round is $O(1)$ during the exploration phase. Moreover, regret-per-round is $\tilde{O}(\Err(T_0/K, T))$ during the exploitation phase. 
The total regret is $T_0 + \tilde{O}(\Err(T_0/K, T)) T$, and optimizing $T_0$ by using the decreasing property of $\Err(N, T)$ in $N$ yields the desired result.
\end{proof}

\subsection{Matching Lower Bound}
We show that this choice of $T_0$ as a function of $T$ is indeed optimal.

\begin{theorem}\label{thm_lower}
Suppose Assumptions \ref{asmp:subgaussian}--\ref{asmp_error}.
Assume that we run the EtC algorithm (Algorithm \ref{alg:esc}). 
    For any choice of $T_0$, the following event occurs with strictly positive probability as $T \to \infty$:
\begin{equation*}
R(T) = \tilde{\Omega}\left( 
\updated{
L(T,3)
}
\right).
\end{equation*}
\end{theorem}
\begin{proof}[Proof sketch of \thmref{thm_lower}]
We explicitly construct a model with $K=3$. Let $\myeps = \Err(T_0/K, T)$. 
In the model, the $\Sigmaarm{1}, \Sigmaarm{2}, \Sigmaarm{3}$ are identical, and the only difference is that the first coefficients $\thetaarmn{1}{1} = 1, \thetaarmn{2}{1} = \Theta(\myeps), \thetaarmn{3}{1} = 0$. All other coefficients are set to zero.
Roughly speaking, the gap is  
\[
\left|
\langle \Xarmt{i}, \hatthetaarm{i} \rangle - \langle \Xarmt{j}, \hatthetaarm{j} \rangle
\right|
=
\left\{
\begin{array}{ll}
\Theta(1) & (i=1, j=2,3)\\
\Theta(\myeps) & (i=2, j=3)
\end{array}
\right.
.
\]
The regret-per-round during the exploration phase is $\Theta(1)$ due to a misidentification of the best arm between arm $1$ and arms $\{2,3\}$. The regret-per-round during the exploitation phase is $\Theta(\myeps)$ due to a misidentification of the best arm between arm $2$ and arm $3$.
\end{proof}
\section{Adaptive Explore-then-Commit (AEtC) Algorithm}

In the prior section, we demonstrated that the optimal way to minimize EtC's regret is by selecting $T_0$, with $T_0$ balancing the exploration and the exploitation. However, this requires knowledge of the covariance matrix's spectrum, which can be difficult to obtain in advance in certain scenarios.
This section explores the way to adaptively determines the extent of exploration required.

\subsection{Estimator}

We assume that $\Sigmai$ follows the data-generating process of Example 1 or 2 in Proposition \ref{prop_examples}.
We use $\beta_T$ to denote that
\begin{equation}
\lambdainum{j} = \Clambda j^{-\beta_T},
\end{equation}
with some constant $\Clambda > 0$.
We have $\beta_T = 1 + 1/T^a$ for Example 1, and $\beta_T = b (<1)$  for Example 2. 

Balancing exploration and exploitation in an overparameterized model is challenging for the following reasons. First, the number of features $p$ is very large or even infinite\footnote{Namely, $p = \infty$ for for Proposition \ref{prop_examples} (1) or $p= T^c$ in for Proposition \ref{prop_examples} (2).}. Second, the trace is heavy-tail because the decay is not very fast.
As a result, a naive use of a traditional method does not work. We address this issue by utilizing two estimators. The first estimator is on the trace $\trace(\hat{\Sigma}^{(i)})$ that we extracted from overparameterization theory \citep{bartlett2020benign}, whereas the second estimator is on the estimated decay rate $\hat{\beta}_T$ that derives from \cite{bosq2000linear,koltchinskii2017concentration}.

For an arm $i \in [K]$ with $N = T_0 / K$ draws, we define an estimated eigenvalues $\hat{\lambda}_1^{(i)},...,\hat{\lambda}_{N}^{(i)} > 0$ from an empirical covariance matrix $\hat{\Sigma}^{(i)} := N^{-1} \sum_{j=1}^{N} \Xin{j}(t) \Xin{j}(t)^\top$. We define the empirical trace to be $\trace(\hatSigmai) = \sum_j \hat{\lambda}_j^{(i)}$.
The following lemma states the consistency of the estimated trace under very mild conditions.
\begin{lemma}{\rm (Error bound of empirical trace)}\label{lem_traceest}
Suppose Assumption \ref{asmp:subgaussian}. Suppose the data generating process (DGP) of Example 1 or Example 2 in Proposition \ref{prop_examples}.
Then, for any $\CPoly>0$, with probability at least $1 - T^{-\CPoly}$, we have the following as $T \to \infty$:
\begin{equation}
\frac{|\trace(\Sigmai) - \trace(\hatSigmai)|}{\trace(\Sigmai)}
= \tilde{O}\left(
\frac{\sqrt{\sum_j (\lambda_j^{(i)})^2} }{\sum_j \lambda_j^{(i)}}
\right)
= o(1). \label{ineq_traceerror}
\end{equation}
\end{lemma}%
Moreover, the following bounds the estimation error of each eigenvalue.
\begin{lemma} \label{lem:eigenvalue}
Suppose Assumption \ref{asmp:subgaussian}.
For any $\delta \in (0,1)$, for any $\CPoly>0$, with probability at least $1 - T^{-\CPoly}$, we have the following for any $j = 1,...,p$ as $T \to \infty$:
\begin{align}\label{ineq_eigenerror}
    \left|\hat{\lambda}_j^{(i)} - \lambda_j^{(i)} \right| = O\left( \sqrt{\frac{\trace(\Sigmai)}{N}} 
    \right).
\end{align}
\end{lemma}
Lemma \ref{lem:eigenvalue} implies the estimator of each eigenvalue is $o(1)$ if we choose $N \gg \trace(\Sigmai)$. However, even if we choose a large $N$, the error is still non-negligible for the tail of eigenvalues where $|\hat{\lambda}_j^{(i)} - \lambda_j^{(i)} |$ is very small\footnote{Remember that we consider $\Sigmai = \Sigmai(T)$ that depends on $T$. Tail of eigenvalues are $o(1)$ to $T$ as well.}.
Consequently, $\hat{\lambda}_j^{(i)}$ for large $j$ are not necessarily consistent, so as to the effective rank and the coherent rank estimated from them. To address this, we estimate the decay rate $\beta_T$, and then estimate the subsequent statistics.

Let the estimated decay rate be $\hat{\beta}_T = (1/\tau)\sum_{k=1}^\tau \log(\hat{\lambda}_k/\hat{\lambda}_{k+1})/\log((k+1)/k)$.
Theoretically, $\hat{\beta}_T$ is consistent for any constant $\tau > 1$. In practice, we can use a reasonable constant $\tau$, such as $\tau = 10$, for robustness.
To estimate the effective ranks, we use the following form
\[
\tilde{\lambda}_k^{(i)} = \hat{\lambda}_1^{(i)} k^{-\hat{\beta}_T}.
\]
Namely, we consider empirical analogues of the effective rank for $k$ as
\begin{align*}
    \hat{r}_k(\Sigmai) &:= 
    {\frac{
    \trace(\hatSigmai)
    }{\tilde{\lambda}_{k+1}^{(i)}},}
    ~~\mbox{and}~~
    \hat{R}_k(\Sigmai) := 
    {\frac{(
    \trace(\hatSigmai)
    )^2}{\sum_{j > k} (\tilde{\lambda}_j^{(i)})^2}}.
\end{align*}
We also define an estimator for the coherent rank as $\hat{k}_{N} := \min \{k \in \N \cup \{0\} \mid  \hat{r}_k(\Sigmai) \geq N\}$.
Further, we define estimators of $B_{N,T}^{(i)}$, $V^{(i)}_{N,T}$ of Eq.~\eqref{def:effective_bias_var} as
\begin{align*}
     \hat{B}_{N,T}^{(i)} := \tilde{\lambda}^{(i)}_{\hat{k}_{N}},\mbox{~and~}
     \hat{V}_{N,T}^{(i)} := \left( \frac{\hat{k}_{N}}{N} + \frac{N}{ \hat{R}_{\hat{k}_{N}}(\Sigmai)}  \right),
\end{align*}
Note that the estimators $\hat{r}_k$ and $\hat{R}_k$ use the trace $\trace(\hatSigmai)$, which is a sum of eigenvalues, instead of the partial sum of the eigenvalues that constitutes the effective rank of Eq.~\eqref{def:effective_rank}. 
Despite this change, this does not affect\footnote{This is because $\sum_j j^{-\beta_T}$ with $\beta_T <1 $ or $\beta_T \approx 1$ is tail-heavy.} asymptotic consistency of the coherent rank estimator $\hat{k}_{N}$.

The following lemma states that small $\tau$ suffices to assure the quality of $\hat{\beta}^{(i)}$.
We study a convergence rate of $\hat{\beta}^{(i)}$ and the other estimators as follows.
\begin{lemma}\label{lem_alphaest}
Suppose Assumption \ref{asmp:subgaussian}. Suppose DGP of Example 1 or Example 2 in Proposition \ref{prop_examples}.\footnote{Note that DGP of Example 1 or 2 implies the existence of the error function of Assumption \ref{asmp_error}.}
Assume that we choose $N = N(T)$ such that $\trace(\Sigmai)/N = o(1)$ for all $i$.
Then, for any $\CPoly>0$, with probability at least $1 - 1/T^{-\CPoly}$,
we have 
\[
\frac{|\tilde{\lambda}_k^{(i)} - \lambdaik|}{\lambdaik} = o(1),
\]
as $T \to \infty$ for all $k \in [\tau]$.
Moreover, it implies $|\beta_T - \hat{\beta}_T| = o(1)$
and 
\begin{equation}\label{ineq_statsconsistency}
\max\left\{\frac{\hat{B}_{N,T}^{(i)} }{ B_{N,T}^{(i)} },
\frac{\hat{V}_{N,T}^{(i)} }{ V_{N,T}^{(i)} } \right\}
= T^{o(1)}.
\end{equation}
\end{lemma}
Eq.~\eqref{ineq_statsconsistency} states that the estimated rate of error is accurate as $T\rightarrow \infty$. To see this, for Example 1 in Proposition \ref{prop_examples}, the estimated error rate is $T^{o(1)}(\sqrt{T^{a}/N + T^{-a}})$, whose exponent approaches to $\sqrt{T^{a}/N  + T^{-a}}$ as $T \rightarrow \infty$.

\subsection{Adaptive Explore-then-Commit (AEtC)}

\begin{algorithm}[t]
\caption{Adaptive Explore-then-Commit (AEtC)}\label{alg:unknown_alpha}
\begin{algorithmic}
\WHILE{$t = 1,2,3,...$}
    \STATE Observe $\Xit$ for all $i \in [K]$.
    \STATE Choose $\Ich = t - K\lfloor t/K \rfloor$.
    \STATE Receive a reward $\Yt$.
    \IF{$t \in \{ \lfloor e^N \rfloor \mid N \in \N, N \ge \log T\}$} 
        \IF{$\Stop(t/K)$}
            \STATE $T_0 \leftarrow t$.
            \STATE Break.
        \ENDIF
    \ENDIF
\ENDWHILE
\FOR{$i \in [K]$}
    \STATE $\hatthetai \leftarrow (\bfXi)^\top (\bfXi (\bfXi)^\top )^{-1} \bfYi$.
\ENDFOR
\FOR{$t = T_0 + 1,...,T$}
    \STATE Observe $\Xit$ for all $i \in [K]$.
    \STATE Choose $\Ich = \argmax_{i \in [K]}  \langle \Xit, \hatthetai \rangle$.
    \STATE Receive a reward $\Yt$.
\ENDFOR
\end{algorithmic}
\end{algorithm}

The Adaptive Explore-then-Commit (AEtC) algorithm is described in \algoref{alg:unknown_alpha}. 
Unlike EtC, the amount of exploration in AEtC is adaptively determined based on the stopping condition:
\[
\Stopi(N) = \left\{
N > \CStopTr\,\trace(\hatSigmai) \cap 
NK \ge T \sqrt{\hat{B}_{N,T}^{(i)} + \hat{V}_{N,T}^{(i)}}
\right\},
\]
and $\Stop = \cap_{i\in[K]} \Stopi$, where $\CStopTr = \CStopTr(T)$ is a function of $T$ that slowly diverges to $+\infty$.
The first condition $N > \CStopTr\,\trace(\hatSigmai)$ ensures that $N$ is large enough to have consistency on the estimators, whereas the second condition $NK \ge T \sqrt{\hat{B}_{N,T}^{(i)} + \hat{V}_{N,T}^{(i)}}$ balances the exploration and exploitation.
The following theorem bounds the regret of AEtC.
\begin{theorem} \label{thm:main2}
Suppose Assumption \ref{asmp:subgaussian}. Suppose DGP of Example 1 or Example 2 in Proposition \ref{prop_examples}.
Then, for any $c>0$, the regret \eqref{def:regret} when we run the AEtC algorithm (\algoref{alg:unknown_alpha}) with a sufficiently slowly diverging $\CStopTr$ is bounded as follows as $T \to \infty$:
\updated{
\begin{align*}
    R(T)  = \tilde{O}(\updated{L(T,K)}T^{c}),
\end{align*}
where $L(T,K)$ is the function defined in Theorem \ref{thm:upper_bound}.
}
\end{theorem}
Theorem \ref{thm:main2} states that, we can choose arbitrarily small $c>0$ and the exponent of the regret of AEtC approaches that of EtC (Theorem \ref{thm:upper_bound}). 
Note that the condition $N \ge \CStopTr \trace(\hatSigmai)$ should not dominate the balance between exploration and exploitation: For example, $\CStopTr = C \log T$ for any $C>0$ suffices. In practice, a constant $\CStopTr$ works.
\begin{proof}[Proof sketch of Theorem \ref{thm:main2}]
We can derive that Eqs.~\eqref{ineq_traceerror}, ~\eqref{ineq_eigenerror}, and \eqref{ineq_statsconsistency} for all $N \in [T]$ that hold with probability at least $1-O(1/T)$ by setting $\CPoly = 2$ and taking a union bound over $N$, and thus
\begin{align}
\bigcap_{i\in[K]} \bigcap_{N \ge \CStopTr\trace(\hatSigmai)} 
\left\{
\log\left(
\frac{\sqrt{\hat{B}_{N,T}^{(i)} + \hat{V}_{N,T}^{(i)}}}{\sqrt{B_{N,T}^{(i)} + V_{N,T}^{(i)}}}
\right)
= \log(T^{o(1)})
\right\}
\end{align}
holds. Under this event, the stopping time of AEtC and EtC are at most $T^{o(1)}$ times different.
Given this, the proof of the theorem is easily modified from the proof of Theorem \ref{thm:upper_bound}.
\end{proof}

\begin{figure*}[htbp]
\centering
\begin{tabular}{cccc}
\includegraphics[width=0.23\textwidth]{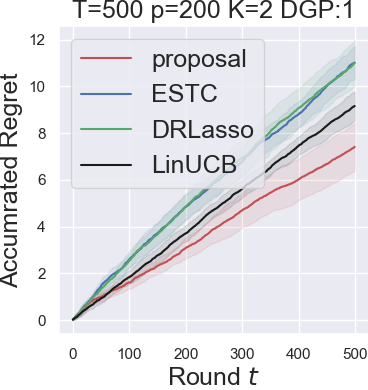} &
\includegraphics[width=0.23\textwidth]{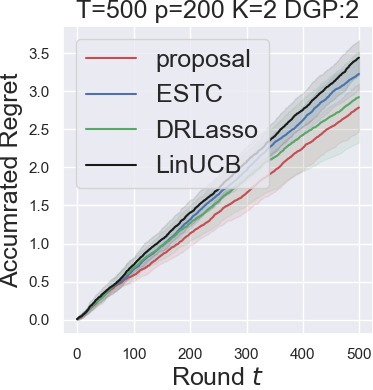} &
\includegraphics[width=0.23\textwidth]{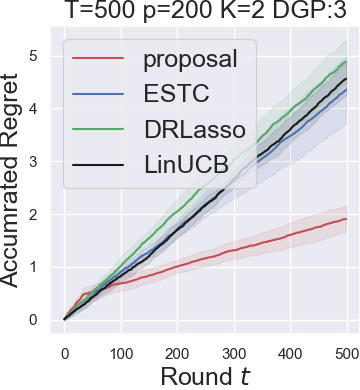} &
\includegraphics[width=0.23\textwidth]{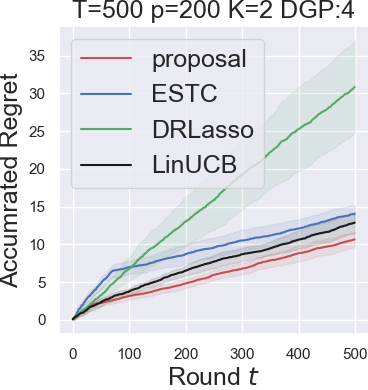} \\
\includegraphics[width=0.23\textwidth]{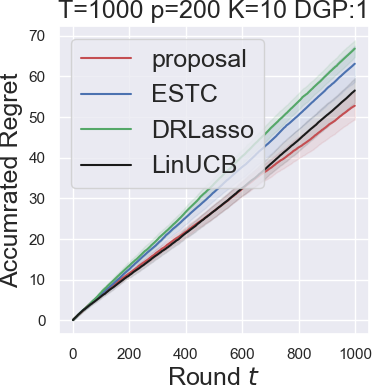} &
\includegraphics[width=0.23\textwidth]{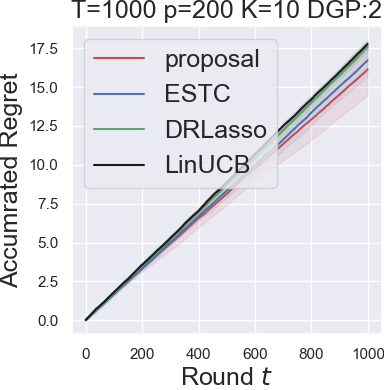} &
\includegraphics[width=0.23\textwidth]{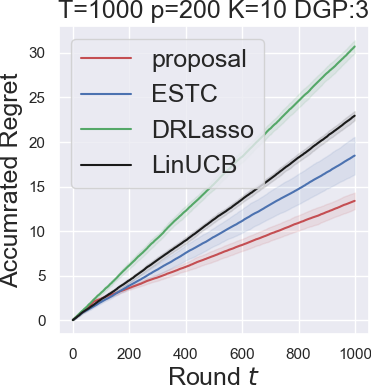} &
\includegraphics[width=0.23\textwidth]{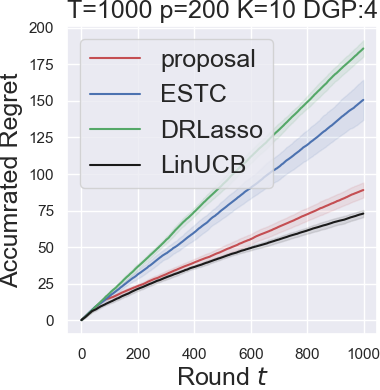} \\
\end{tabular}
\caption{Comparison of methods on four DGPs. Smaller regret indicates better performance.
The solid lines are the mean of $10$ repetitions and the bands represent the standard deviation.
\label{fig:regret}}
\end{figure*}

\section{Simulation}\label{sec_sim}

We consider two setups: $(K,p,T) = (2,200,500)$ and $(K,p,T) = (10,200,1000)$.
For each setup, we consider a covariance matrix $\Sigmai = c^{(i)} \bar{\Sigma}$ where $c^{(i)} \sim \mathrm{Uni}([0.5,1.5])$ and a base covariance $\Bar{\Sigma}$  for each $i \in [K]$, which represents a heterogeneous covariance among the arms.
The base covariance $\Bar{\Sigma}$ follows the following configurations: DGP 1: $\Bar{\Sigma}  = \mathrm{diag}(\lambda_{1},...,\lambda_{p}), \lambda_{k} = k^{-0.5}$, DGP 2: $\Bar{\Sigma} = \mathrm{diag}(\lambda_{1},...,\lambda_{p}), \lambda_{k} = \exp(-k) + T \exp(-T)/p$, DGP 3: $\Bar{\Sigma}  = \mathrm{diag}(\lambda_{1},...,\lambda_{p}), \lambda_{k} = k^{-1 + 1/T} $, and DGP 4: $\Bar{\Sigma}_{i,j} = 0.3$ and $\Bar{\Sigma}_{i,i} = 0.7$ for $j \neq i \in [p]$.
Note that the DGPs 1 and 3 correspond to Examples 2 and 1 in Proposition \ref{prop_examples}. DGP 3 is benign as well, while DGP 4 is not.
We generate $X^{(i)}(t)$ from a centered $p$-dimensional Gaussian with covariance $\Sigmai$, $\thetai$ from a standard normal distribution which yields non-sparse parameter vectors, and $Y^{(i)}(t)$ by the linear model \eqref{eq:linear_model} with the noise with variance $\sigma^2 = 0.01$.
We compare proposal (AEtC, $\CStopTr=2$) to ESTC \citep{hao2020high}, LinUCB \citep{li2010,abbasi2011improved}, and DR Lasso Bandit \citep{kim2019doubly}. 

Figure \ref{fig:regret} illustrates the accumulated regret $R(t)$ in relation to the round $t \in [T]$. In the first three benign DGPs, the AEtC consistently outperforms the other methods. 
In DGP 2, the advantage of AEtC is insignificant because the regret of any method is already small. This is because in DGP 2, a model with exponentially decaying behavior, only a small fraction of eigenvalues are important. However, AEtC performs significantly better than its competitors in DGP 1 and DGP 3, where the eigenvalues have a heavy-tail distribution.
In a non-benign example of DGP 4, the proposed method is still comparable to LinUCB, which demonstrates the utility of an interpolating estimator.

\section*{Acknowledgement}

We thank Shogo Iwazaki for valuable comments on revising the manuscript.

\bibliographystyle{plainnat}
\bibliography{main}

\begin{thebibliography}{29}
\providecommand{\natexlab}[1]{#1}
\providecommand{\url}[1]{\texttt{#1}}
\expandafter\ifx\csname urlstyle\endcsname\relax
  \providecommand{\doi}[1]{doi: #1}\else
  \providecommand{\doi}{doi: \begingroup \urlstyle{rm}\Url}\fi

\bibitem[Abbasi-Yadkori et~al.(2011)Abbasi-Yadkori, P{\'a}l, and
  Szepesv{\'a}ri]{abbasi2011improved}
Yasin Abbasi-Yadkori, D{\'a}vid P{\'a}l, and Csaba Szepesv{\'a}ri.
\newblock Improved algorithms for linear stochastic bandits.
\newblock \emph{Advances in neural information processing systems}, 24, 2011.

\bibitem[Abe and Long(1999)]{abe1999}
Naoki Abe and Philip~M. Long.
\newblock Associative reinforcement learning using linear probabilistic
  concepts.
\newblock In \emph{Proceedings of the Sixteenth International Conference on
  Machine Learning}, ICML, page 3^^e2^^80^^9311. Morgan Kaufmann Publishers
  Inc., 1999.

\bibitem[Agarwal et~al.(2012)Agarwal, Chen, Elango, and
  Wang]{Agarwal2012PersonalizedCS}
Deepak~K. Agarwal, Bee-Chung Chen, Pradheep Elango, and Xuanhui Wang.
\newblock Personalized click shaping through lagrangian duality for online
  recommendation.
\newblock In \emph{SIGIR '12}, 2012.

\bibitem[Auer et~al.(2002)Auer, Cesa{-}Bianchi, and Fischer]{AuerCF02}
Peter Auer, Nicol{\`{o}} Cesa{-}Bianchi, and Paul Fischer.
\newblock Finite-time analysis of the multiarmed bandit problem.
\newblock \emph{Mach. Learn.}, 47\penalty0 (2-3):\penalty0 235--256, 2002.
\newblock \doi{10.1023/A:1013689704352}.

\bibitem[Bartlett et~al.(2020)Bartlett, Long, Lugosi, and
  Tsigler]{bartlett2020benign}
Peter~L Bartlett, Philip~M Long, G{\'a}bor Lugosi, and Alexander Tsigler.
\newblock Benign overfitting in linear regression.
\newblock \emph{Proceedings of the National Academy of Sciences}, 117\penalty0
  (48):\penalty0 30063--30070, 2020.

\bibitem[Bastani and Bayati(2020)]{BastaniB20}
Hamsa Bastani and Mohsen Bayati.
\newblock Online decision making with high-dimensional covariates.
\newblock \emph{Oper. Res.}, 68\penalty0 (1):\penalty0 276--294, 2020.

\bibitem[Bosq(2000)]{bosq2000linear}
Denis Bosq.
\newblock \emph{Linear processes in function spaces: theory and applications},
  volume 149.
\newblock Springer Science \& Business Media, 2000.

\bibitem[Chakraborty and Murphy(2014)]{Chakraborty2014}
Bibhas Chakraborty and Susan~A. Murphy.
\newblock Dynamic treatment regimes.
\newblock \emph{Annual Review of Statistics and Its Application}, 1\penalty0
  (1):\penalty0 447--464, 2014.

\bibitem[Chu et~al.(2011)Chu, Li, Reyzin, and Schapire]{chu11a}
Wei Chu, Lihong Li, Lev Reyzin, and Robert Schapire.
\newblock Contextual bandits with linear payoff functions.
\newblock In Geoffrey Gordon, David Dunson, and Miroslav Dud^^c3^^adk, editors,
  \emph{Proceedings of the Fourteenth International Conference on Artificial
  Intelligence and Statistics}, volume~15 of \emph{Proceedings of Machine
  Learning Research}, pages 208--214. PMLR, 11--13 Apr 2011.

\bibitem[Foster et~al.(2019)Foster, Krishnamurthy, and
  Luo]{DBLP:conf/nips/FosterKL19}
Dylan~J. Foster, Akshay Krishnamurthy, and Haipeng Luo.
\newblock Model selection for contextual bandits.
\newblock In \emph{Advances in Neural Information Processing Systems 32}, pages
  14714--14725, 2019.

\bibitem[Hao et~al.(2020)Hao, Lattimore, and Wang]{hao2020high}
Botao Hao, Tor Lattimore, and Mengdi Wang.
\newblock High-dimensional sparse linear bandits.
\newblock \emph{Advances in Neural Information Processing Systems},
  33:\penalty0 10753--10763, 2020.

\bibitem[Hara and Maehara(2017)]{hara2017enumerate}
Satoshi Hara and Takanori Maehara.
\newblock Enumerate lasso solutions for feature selection.
\newblock In \emph{Proceedings of the AAAI Conference on Artificial
  Intelligence}, volume~31, 2017.

\bibitem[Jang et~al.(2022)Jang, Zhang, and Jun]{JangZJ22}
Kyoungseok Jang, Chicheng Zhang, and Kwang{-}Sung Jun.
\newblock Popart: Efficient sparse regression and experimental design for
  optimal sparse linear bandits.
\newblock In \emph{NeurIPS}, 2022.

\bibitem[Kim and Paik(2019)]{kim2019doubly}
Gi-Soo Kim and Myunghee~Cho Paik.
\newblock Doubly-robust lasso bandit.
\newblock \emph{Advances in Neural Information Processing Systems}, 32, 2019.

\bibitem[Koltchinskii and Lounici(2017)]{koltchinskii2017concentration}
Vladimir Koltchinskii and Karim Lounici.
\newblock Concentration inequalities and moment bounds for sample covariance
  operators.
\newblock \emph{Bernoulli}, 23\penalty0 (1):\penalty0 110--133, 2017.

\bibitem[Lai and Robbins(1985)]{Lairobbins1985}
T.L Lai and Herbert Robbins.
\newblock Asymptotically efficient adaptive allocation rules.
\newblock \emph{Advances in Applied Mathematics}, 6\penalty0 (1):\penalty0
  4--22, 1985.
\newblock ISSN 0196-8858.

\bibitem[Li et~al.(2010)Li, Chu, Langford, and Schapire]{li2010}
Lihong Li, Wei Chu, John Langford, and Robert~E. Schapire.
\newblock A contextual-bandit approach to personalized news article
  recommendation.
\newblock In \emph{Proceedings of the 19th International Conference on World
  Wide Web}, WWW '10, page 661^^e2^^80^^93670. Association for Computing
  Machinery, 2010.

\bibitem[Li et~al.(2022)Li, Barik, and Honorio]{li2022simple}
Wenjie Li, Adarsh Barik, and Jean Honorio.
\newblock A simple unified framework for high dimensional bandit problems.
\newblock In \emph{International Conference on Machine Learning}, pages
  12619--12655. PMLR, 2022.

\bibitem[Miolane and Montanari(2021)]{miolane2021distribution}
L{\'e}o Miolane and Andrea Montanari.
\newblock The distribution of the lasso: Uniform control over sparse balls and
  adaptive parameter tuning.
\newblock \emph{The Annals of Statistics}, 49\penalty0 (4):\penalty0
  2313--2335, 2021.

\bibitem[Nakakita and Imaizumi(2022)]{nakakita2022benign}
Shogo Nakakita and Masaaki Imaizumi.
\newblock Benign overfitting in time series linear model with
  over-parameterization.
\newblock \emph{arXiv preprint arXiv:2204.08369}, 2022.

\bibitem[Oh et~al.(2021)Oh, Iyengar, and Zeevi]{oh2021sparsity}
Min-hwan Oh, Garud Iyengar, and Assaf Zeevi.
\newblock Sparsity-agnostic lasso bandit.
\newblock In \emph{International Conference on Machine Learning}, pages
  8271--8280. PMLR, 2021.

\bibitem[Rendle(2010)]{Rendle10}
Steffen Rendle.
\newblock Factorization machines.
\newblock In Geoffrey~I. Webb, Bing Liu, Chengqi Zhang, Dimitrios Gunopulos,
  and Xindong Wu, editors, \emph{{ICDM} 2010, The 10th {IEEE} International
  Conference on Data Mining, Sydney, Australia, 14-17 December 2010}, pages
  995--1000. {IEEE} Computer Society, 2010.

\bibitem[Robbins(1952)]{robbins1952}
Herbert Robbins.
\newblock {Some aspects of the sequential design of experiments}.
\newblock \emph{Bulletin of the American Mathematical Society}, 58\penalty0
  (5):\penalty0 527 -- 535, 1952.

\bibitem[Tang et~al.(2013)Tang, Rosales, Singh, and Agarwal]{tang13}
Liang Tang, R{\'{o}}mer Rosales, Ajit Singh, and Deepak Agarwal.
\newblock Automatic ad format selection via contextual bandits.
\newblock In \emph{22nd {ACM} International Conference on Information and
  Knowledge Management, CIKM'13}, pages 1587--1594. {ACM}, 2013.

\bibitem[Tsigler and Bartlett(2020)]{tsigler2020benign}
Alexander Tsigler and Peter~L Bartlett.
\newblock Benign overfitting in ridge regression.
\newblock \emph{arXiv preprint arXiv:2009.14286}, 2020.

\bibitem[Van De~Geer and B{\"u}hlmann(2009)]{van2009conditions}
Sara~A Van De~Geer and Peter B{\"u}hlmann.
\newblock On the conditions used to prove oracle results for the lasso.
\newblock \emph{Electronic Journal of Statistics}, 3:\penalty0 1360--1392,
  2009.

\bibitem[Vershynin(2018)]{vershynin2018high}
Roman Vershynin.
\newblock \emph{High-dimensional probability: An introduction with applications
  in data science}, volume~47.
\newblock Cambridge university press, 2018.

\bibitem[Wang et~al.(2018)Wang, Wei, and Yao]{wang2018minimax}
Xue Wang, Mingcheng Wei, and Tao Yao.
\newblock Minimax concave penalized multi-armed bandit model with
  high-dimensional covariates.
\newblock In \emph{International Conference on Machine Learning}, pages
  5200--5208. PMLR, 2018.

\bibitem[Wang et~al.(2022)Wang, Tao, and Zhang]{yuyan2022}
Yuyan Wang, Long Tao, and Xian~Xing Zhang.
\newblock Recommending for a multi-sided marketplace with heterogeneous
  contents.
\newblock In \emph{Proceedings of the 16th ACM Conference on Recommender
  Systems}, RecSys '22, page 456^^e2^^80^^93459. Association for Computing
  Machinery, 2022.

\end{thebibliography}

\newpage
\appendix
\onecolumn

\section{Limitations}

The following characterizes the limitations. We consider addressing these as interesting directions for future work.

\begin{itemize}[leftmargin=*]
\item \textbf{More adaptive algorithms, such as ones based on upper confidence bounds:} This paper considers the class of the explore-then-commit methods. In many bandit problems, the upper confidence bound (UCB) method provides better empirical performance since it adaptively balances exploration and exploitation. Applying UCB to this problem is important future work. 
The key challenge here is that such an adaptive estimator involves a biased selection of the context vectors, which requires a more adaptive error bound for a high-dimensional linear model, such as the self-normalized bound \citep{abbasi2011improved}.
\item \textbf{Lower bound of the problem:} While we showed the optimal choice of $T_0$ by deriving a matching bound, this does not exclude the possibility of a more adaptive bandit algorithm (e.g., UCB) that achieves a better rate of regret.
Explicit construction of the lower bound requires two different models where the probability of an accurate (low-regret) estimate in one model implies a misestimation in the other model. 
To our knowledge, such a process for our non-sparse high-dimensional model is challenging because, unlike the sparse bandit model where only a small amount of parameters are active, in the non-sparse regime, we need to devise two models where very large ($p > n$) number of non-equal coefficients and need to bound the KL divergence between such a large multivariate Gaussians.
\item \textbf{Temporal correlation:} This paper assumes temporal indepenence (i.e., $\Xit, X_i(t')$ are independent between $t \ne t'$). While such an assumption is popular, allowing the temporal correlation widens the application of the framework. For example, the click probability of online advertisements depends on the period of time. There are some recent results that are potentially applicable in non-sparse high-dimensional regime (e.g., \citet{nakakita2022benign}). 
\item \textbf{Exponentially decaying models:} 
\updated{
AEtC directly assumes the eigenvalue structure of the covariance matrix (in particular, Example 1 or Example 2 in Proposition \ref{prop_examples}).
}
Exploring the inclusion of exponentially decaying eigenvalues, such as Example 4 in Theorem 31 of \cite{bartlett2020benign}, would be an intriguing avenue to explore.
\item \updated{\textbf{Nested models:}
\cite{DBLP:conf/nips/FosterKL19} considered the contextual bandit problem under the nested model and introduced an algorithm that can adapt to the problem's complexity. Considering such a nested model should be interesting.
}
\end{itemize}

\section{\updated{Comparison with Sparse Bandits}}

\updated{
Similar to our setting, sparse linear bandit models accept a very large number of features\footnote{Typically, the number of feature $p$ can be exponential to the number of datapoints $T$.}. Sparse bandit algorithms employ the $\ell 1$ regularizer to suppress most coefficients. The regret of a sparse bandit algorithm is characterized by the number of non-zero features $s$. In our case, $s$ corresponds to $p$ since all dimensions exhibit non-zero values. Consequently, regret bounds in sparse bandit translate into $O(T)$ trivial bounds within our framework.
We consider the sparse and benign overfitting to be orthogonal. There are numerous examples that are benign yet not sparse, and conversely, sparse but not benign. The primary incentive for benign overfitting theory is to illuminate the learnability of recent large-scale models. To rephrase, we are examining different problem classes in which we can ensure sublinear regret bounds, irrespective of any assumptions about sparsity.
}

\section{Proofs on Risk Bounds}

We give proofs on the upper bound in Theorem \ref{thm_ub}.
The bound is derived from the risk bound in \citet{tsigler2020benign} and adapted for our setting. 
A significant difference here is that the budget $T$, which characterizes the covariance matrices $\Sigmai$, and the sample size $N$ used for estimation have different values.

We give some additional notation.
For a vector $b \in \R^p$ and $q \in [p]$, let $b_{0:q} := (b_1,b_2,...,b_{q})$ is a sub-vector.
Similarly, $b_{q:\infty} := (b_{q+1}, b_{q+2},...,b_p)$ is a sub-vector with the rest of the terms.
For a covariance matrix $\Sigma \in \R^{p \times p}$ with eigenvalues $\lambda_1,...,\lambda_p$, $\Sigma_{0:q} = \mathrm{diag}(\lambda_1,...,\lambda_{q})$ and $\Sigma_{q:\infty} = \mathrm{diag}(\lambda_{q+1},...,\lambda_{p})$ are diagonal matrices with the subset of eigenvalues.
Similarly, for the data matrix $\bfXi \in \R^{N \times p}$, $\bfXi_{0:q}$ denotes a sub-matrix with the first $q$ columns of $\bfXi$, and $\bfXi_{q:\infty}$ denotes a sub-matrix with the rest of the columns $\bfXi$.

\begin{proof}[Proof of Theorem \ref{thm_ub}]

For each $i \in [K]$, we study the bound on the risk $\|\hatthetai - \thetai\|_\Sigmai^2$.
Using the fact that $\bfYi = \bfXi \thetai + \Xi^{(i)}$ with $\Xi^{(i)} = (\xi^{(i)}(1),...,\xi^{(i)}(N))^\top \in \R^{N}$ where $\xi^{(i)}(t)$ is an i.i.d. copy of $\xi^{(i)}(\updated{1})$, we first decompose $\hatthetai$ as
\begin{align*}
    \hatthetai &= (\bfXi)^\top (\bfXi (\bfXi)^\top )^{-1} \bfYi \\
    &= (\bfXi)^\top (\bfXi (\bfXi)^\top )^{-1} \bfXi \thetai + (\bfXi)^\top (\bfXi (\bfXi)^\top )^{-1}\Xi^{(i)} \\
    &=: \Tilde{\theta}^{(i)} + \check{\theta}^{(i)}.
\end{align*}
Using the decomposition, we decompose the total error as
\begin{align*} 
    \|\Tilde{\theta}^{(i)} + \check{\theta}^{(i)} - \thetai\|_\Sigmai^2 \leq 2  \|\Tilde{\theta}^{(i)}  - \thetai\|_\Sigmai^2 + 2\| \check{\theta}^{(i)}\|_\Sigmai^2 =: 2T_B + 2T_V.
\end{align*}
Here, $T_B$ denotes a bias term of the risk, and $T_V$ denotes a variance term.

In the following, we develop a bound on each of the terms.
Fix $k \in [p]$.
By Corollary 6 in \citet{tsigler2020benign}, we achieve the following bounds with some constant $c > 0$, which holds with the desired probability.
Here, we set $\mu_{n,k}(A_k^{-1})$ is the $n$-th largest eigenvalue of $A_k^{-1}$ for $n \in [N]$.
\begin{align}
    &T_B \leq c \|\thetai_{k:\infty}\|_{\Sigmai_{k,\infty}}  + c\|\thetai_{0:k}\|_{(\Sigmai_{0:k})^{-1}} \left( \frac{ \sum_{j > k} \lambdainum{j}}{N} \right)^2, \mbox{~and~} \label{ineq:bias} \\
    &T_V \leq c \frac{k}{N} + c \frac{N \sum_{j > k} (\lambdainum{j})^2}{(\sum_{j > k} \lambdainum{j})^2}. \label{ineq:variance}
\end{align}
Note that only the number of samples $N$ appears explicitly in the boundary, while the budget $T$ affects it only implicitly through $\Sigmai$.

We study the bound for the bias part in \eqref{ineq:bias}.
Here, we set \updated{$k = k^*_N(\Sigmai, N)$} using the coherent rank.
By H\"older's inequality, we obtain
\begin{align*}
T_B &\leq \|\thetai_{k^*_N:\infty}\|_{\Sigmai_{k^*_N:\infty}}^{2}+\|\thetai_{0:k^*_N}\|_{\Sigmai_{0:k^*_N}^{-1}}^{2}\left(\frac{\sum_{j>k^*_N}\lambdainum{j}}{N}\right)^{2}\\
&\le \|\thetai\|^{2}\left(\lambdainum{k^*_N+1}+(\lambdainum{k^*_N})^{-1}\left(\frac{\sum_{j> k^*_N}\lambdainum{j}}{N}\right)^{2}\right)\\
&= \|\thetai\|^{2}\left(\lambdainum{k^*_N+1}+(\lambdainum{k^*_N})^{-1}\left(\frac{\sum_{j\ge k^*_N}\lambdainum{j}-\lambdainum{k^*_N}}{N}\right)^{2}\right)\\
&=\|\thetai\|^{2}\left(\lambdainum{k^*_N+1}+(\lambdainum{k^*_N})^{-1}(\lambdainum{k^*_N})^{2}\left(\frac{\sum_{j\ge k^*_N}\lambdainum{j}-\lambdainum{k^*_N}}{\lambdainum{k^*_N}N}\right)^{2}\right)\\
&=\|\thetai\|^{2}\left(\lambdainum{k^*_N+1}+\lambdainum{k^*_N}\left(\frac{r_{k^*_N-1}(\Sigmai)-1}{N}\right)^{2}\right),
\end{align*}
which follows the definition of the effective rank $r_k(\Sigmai)$.
We also apply the properties of the ranks and obtain
\begin{align*}
    \|\thetai\|^{2}\left(\lambdainum{k^*_N+1}+\lambdainum{k^*_N}\left(\frac{r_{k^*_N-1}(\Sigmai)-1}{N}\right)^{2}\right)
    &\le \|\thetai\|^{2}\left(\lambdainum{k^*_N+1}+\lambdainum{k^*_N}\left(\frac{N-1}{N}\right)^{2}\right)\\
    &\le  \|\thetai\|^{2}\left(\lambdainum{k^*_N+1}+\lambdainum{k^*_N}\right) \\
    & \leq c' \|\thetai\|^{2} \lambdainum{k^*_N} \\
    & \leq c' \|\thetai\|^{2} \biasiNT,
\end{align*}
where $c' := (1 + b^2)$ is a constant.
The second last inequality follows $\lambdainum{k^*_N + 1} \leq \lambdainum{k^*_N}$.
About the variance term $T_V$ in \eqref{ineq:variance}, we simply obtain $T_V \leq \variNT$ by their definitions.

Finally, we set $\CUpper = \max\{c' \|\thetai\|^2, c\}$ and obtain the statement.
\end{proof}

\section{Proofs for EtC}

\subsection{Concentration of the Maximum Value}

\begin{lemma}\label{lem_subgaussimax}
Let $Y_1, Y_2, \dots, Y_K$ be $K$ sub-Gaussian random variables with common parameter $\sigma$. 
Let $Y = \max_{i\in[K]} Y_i$ be the maximum of them\footnote{These random variables can be dependent each other.}. 
Then, we have
\begin{equation}
\Ep[Y] \le \sigma\sqrt{2 \log K}.
\end{equation}
\end{lemma}
\begin{proof}[Proof of Lemma \ref{lem_subgaussimax}]
For any $\lambda > 0$, we have
\begin{align}
e^{\lambda \Ep[\max_{i\in[K]} Y_i]}
&\le \Ep\left[e^{\lambda \max_{i\in[K]} Y_i}\right]
\text{\ \ \ \ (by Jensen's inequality)}
\\
&=\Ep\left[\max_{i\in[K]} e^{\lambda Y_i}\right]\\
&\le \sum_i \Ep[e^{\lambda Y_i}] \le N e^{\lambda^2 \sigma^2 / 2},
\end{align}
and thus 
\[
\Ep[\max_{i\in[K]} Y_i] 
\le \frac{\log N}{\lambda} + \frac{\lambda \sigma^2}{2},
\]
and taking $\lambda = \sqrt{2 \log N / \sigma^2}$ yields the result.
\end{proof} 
\subsection{High-probability Confidence Bound}

\begin{lemma} \label{lem:li}
There exists $C >0$ such that, for any $T_0 \ge CK \log T$,
with probability at least $1-O(T^{-1})$, event
\begin{align}
\mA &= \bigcap_{i\in[K]} \mA_i\\
\mA_i &=
\left\{
\|\hatthetai - \thetai\|_{\Sigmai}
\le 
C \Err(T_0/K, T)
\right\}.
\end{align}

holds. Moreover, under $\mA$, we have
    \begin{align}\label{ineq_maxrad}
        &\Ep[\langle \Xarmt{i^*(t)}, \thetaarm{i^*(t)} \rangle - \langle \Xarmt{\Ich}, \thetaarm{\Ich} \rangle] 
        \le
        2 C \Err(T_0/K, T) \sqrt{2 \log K} 
    \end{align}
    for each $t=T_0+1,T_0+2,\dots,T$.
\end{lemma}

Note that, in EtC, we have uniform exploration over $K$ arms, and thus $N = T_0/K$. 

\begin{proof}[Proof of Lemma \ref{lem:li}]
Theorem \ref{thm_ub} implies that for some $C>0$, with probability at least $1 - \cUpper  e^{-(C\log T)/\cUpper} \ge 1 - 1/T$, $\mA_i$ holds. Moreover, the union bound of this over $K$ arms implies $\mA$ holds with probability $1- K/T$.

Using the definition of the optimal arm $i^*(t) = \argmax_{i^*} \langle X_{i^*}(t), \theta_{i^*} \rangle$,
we have
\begin{align*}
&\langle \Xarmt{i^*(t)}, \thetaarm{i^*(t)} \rangle - \langle \Xarmt{\Ich}, \thetaarm{\Ich} \rangle\\
&\leq 
\langle \Xarmt{i^*(t)}, \hat{\theta}_{i^*(t)} \rangle - \langle \Xarmt{\Ich}, \hat{\theta}_{I(t)}\rangle + 2 D_{I(t)}(t)
\\
&\leq 2 D_{I(t)}(t)
\text{\ \ \ \ (by definition of $I(t)$)}\\
&= 2 \max_{i\in[K]} D_i(t),
\end{align*}
where $\Delta_i = \thetai - \hatthetai$ and $D_i(t) = \langle X_i{(t)}, \Delta_i \rangle$.
Given $\Delta_i$ at the end of round $T_0$, $D_i(t)$ for each $i$ is a sub-Gaussian random variable. 
Under the event $\mA$, the variance of $D_i(t)$ is bounded as
\begin{align}
\Ep[|D_i(t)|]
&\le C \Err(T_0/K, T). \label{ineq_boundhold}
\end{align}
By using this and \lemref{lem_subgaussimax} on the maximum of $K$ sub-Gaussian random variables, we have,
\begin{align}
2 \Ep[|D(t)|] 
&= 2 \Ep[|\max_i D_i(t)|]\\
&\le 2 C \Err(T_0/K, T) \sqrt{2 \log K}.
\text{\ \ \ \ (by \lemref{lem_subgaussimax})}
\end{align}
\end{proof}

\subsection{Proof of Theorem \ref{thm:upper_bound}}

By using Lemma \ref{lem:li}, we derive the regret upper bound of Theorem \ref{thm:upper_bound}.

\begin{proof}[Proof of Theorem \ref{thm:upper_bound}]
The regret (Eq.~\eqref{def:regret}) is bounded as
\begin{align*}
    R(T) &= \Ep \left[ \sum_{t=1}^T  \langle \Xarmt{i^*(t)}, \thetaarm{i^*(t)} \rangle - \langle \Xarmt{\Ich}, \thetaarm{\Ich} \rangle \right] \\
    &= \sum_{t=1}^{T_0} \Ep[\langle \Xarmt{i^*(t)}, \thetaarm{i^*(t)} \rangle - \langle \Xarmt{\Ich}, \thetaarm{\Ich} \rangle]\\
    & \quad + \sum_{t=T_0 + 1}^T \Ep[\langle \Xarmt{i^*(t)}, \thetaarm{i^*(t)} \rangle - \langle \Xarmt{\Ich}, \thetaarm{\Ich} \rangle], \\
    &=: R_1 + R_2.
\end{align*}

We bound the first term $R_1$ as 
\begin{align*}
    R_1 & \leq \sum_{t=1}^{T_0} 2 \Ep\left[\max_{i \in [K]} \langle \Xit, \thetai \rangle \right] \\
    & \leq \sum_{t=1}^{T_0} 2 \kappa_x \sqrt{2 \log K} \ThetaScale \\
    & \text{\ \ \ \ (by Assumption \ref{asmp:subgaussian} and Lemma \ref{lem_subgaussimax})}\\
    &= T_0 \times 2 \kappa_x \ThetaScale \sqrt{2 \log K}  = \tilde{O}(T_0).
\end{align*}

We bound the second term $R_2$ as
\begin{align*}
    R_2 & \leq \sum_{t=T_0 + 1}^T \Ep\left[\mone\{\mA\} 2 C\Err(T_0/K, T) \sqrt{2 \log K}  + \mone\{\mA^c\} (\langle \Xarmt{i^*(t)}, \thetaarm{i^*(t)} \rangle - \langle \Xarmt{\Ich}, \thetaarm{\Ich} \rangle)\right]  \\
    & \text{\ \ \ \ (by Eq.~\eqref{ineq_maxrad})}\\
    &\leq  \sum_{t=T_0 + 1}^T \Pr(\mA) 2 C \Err(T_0/K, T) \sqrt{2 \log K}  + \Pr(\mA^c) \times 2 \kappa_x \ThetaScale \sqrt{2 \log K}  \\
    &\ \ \ \text{\ \ \ \ (by the same discussion as $R_1$)}\\
    &\leq 2 T C \Err(T_0/K, T) \sqrt{2 \log K}  + \tilde{O}(T) \Pr(\mA^c)\\
    &\ \ \ \text{\ \ \ \ (by Lemma \ref{lem:li})}\\
    &\leq 2 T C \Err(T_0/K, T) \sqrt{2 \log K} + \tilde{O}(T) \times \frac{K}{T}.
\end{align*}

Combining these bounds, we obtain
\begin{align*}
    R(T) = R_1 + R_2 \leq \tilde{O}(T_0) + \tilde{O}(T \Err(T_0/K, T)).
\end{align*}
This bound is optimized by choosing $T_0$ in accordance with Eq.~\eqref{amount_exploration},
and then we have the theorem.
\end{proof} 

\subsection{Proof of \thmref{thm_lower}}

This section provides the regret lower bound of EtC.

\begin{proof}[Proof of \thmref{thm_lower}]
We consider a model with $K=3$. 
Let $\myeps = \Err(T_0/K, T)$.
We explicitly construct $\Sigmai$ as follows: First, $\Sigmaarm{1}, \Sigmaarm{2}, \Sigmaarm{3}$ are identical and benign, and denote $k$-th eigenvalue as $\lambdaik = \lambda_k$. 
The coefficient $\thetaarm{1} = (1,0,0,\dots)^\top$, $\thetaarm{2} = (\myeps,0,0,\dots)^\top$, and $\thetaarm{3} = (0,0,0,\dots)^\top$. With these $\{\thetai\}_{i=1,2,3}$, the only non-zero coefficients are $\thetaarmn{1}{1}, \thetaarmn{2}{1}$, and thus the best arm is defined in terms of $\Xarmnt{1}{1}, \Xarmnt{2}{1}$, which are the first components of $\Xarmt{1}, \Xarmt{2}$, respectively.

Let $R_1, R_2$ be regret during the exploration and exploitation phases, respectively.

\noindent\textbf{Regret during the exploration phase:}\\
We first bound the regret during the exploration phase where we draw an arm uniformly randomly. In the following, we show that the regret per round is $\Omega(1)$. 
Since $\Xit$ is a zero-mean ($p$-dimensinal) Gaussian, its linear function such as $\langle \Xit, \thetai \rangle$ are $\langle \Xit, \hatthetai \rangle$ are zero-mean univariate Gaussians given $\thetai, \hatthetai$.
Therefore,  
\[
r_{ij}(t) := 
\langle \Xit, \thetai \rangle 
- 
\langle \Xarmt{j}, \thetaarm{j} \rangle 
\]
is also a Gaussian, and its variance is $\Theta(1)$.
Therefore, $\Pr[r_{12}(t) \ge 1], \Pr[r_{13}(t) \ge 1] = \Theta(1)$. The regret in the rounds we draw arm $2,3$ is at least $\mone[r_{12}(t) \ge 1]r_{12}(t), \mone[r_{13}(t) \ge 1]r_{13}(t) = \Theta(1)$, which implies the regret-per-round is $\Omega(1)$.

In summary, $\Ep[R_1] = \Omega(1) \times T_0 = \Omega(T_0)$.

\noindent\textbf{Regret during the exploitation phase:}\\
We then bound the regret during the exploitation phase. Intuitively speaking, an algorithm misidentifies the better arm among $2$ and $3$ with $\Omega(1)$ probability and the regret per such a misidentification is $\Omega(\myeps)$.

Theorems~\ref{thm_ub} and Lemma \ref{lem_lower} state that there exists a constant $c>0$ such that probability at least $1 - c e^{-N/c}$, for all $i$ we have
\begin{equation}\label{ineq_highvar}
\Ep\left[\|\hatthetai - \thetai\|_\Sigmai \right] \in (C^L \myeps, C^U \myeps)
\end{equation}
for some constants $C^L, C^U > 0$. 

Without loss of generality, we assume $N$ to be sufficiently large such that $1 - c e^{-N/c} > 1/2$ because if $N = O(\log T)$ then the size of the error bound is at least polylogarithmic, which results in regret of $\tilde{\Omega}(T)$.
In the following, we assume Eq.~\eqref{ineq_highvar} is the case, and we show that the regret-per-round is $\Omega(\myeps)$. 
We define the events as follows:
\begin{align} 
\mG(t) &:= \{\Xarmnt{2}{1} \in [1, 2]\}\\
\mH(t) &:= \{\langle \Xarmnt{1}{1}, \hatthetaarm{1} \rangle< \langle \Xarmnt{2}{1}, \hatthetaarm{2} \rangle <  \langle \Xarmnt{3}{1}, \hatthetaarm{3} \rangle \}.
\end{align}
Event $\mG(t)$ states that arm $2$ is positive and not very large, Event $\mH(t)$ states that the algorithm draws considers the arm $3$ is the best and arm $2$ is the second best.

If $\mG(t) \cap \mH(t)$ is the case, 
then arm $3$ is chosen but suboptimal,
and the regret is at least
\begin{align}
\langle \Xarmnt{2}{1}, \thetaarmn{2}{1} \rangle -  \langle \Xarmnt{3}{1}, \thetaarmn{3}{1} \rangle
\in [\myeps, 2 \myeps], \label{ineq_lower_regretperround}
\end{align}
where we used $\mG(t)$, $\langle \Xarmt{2}, \thetaarm{2} \rangle =  \Xarmnt{2}{1} \myeps$, and $\langle \Xarmt{3}, \thetaarm{3} \rangle = 0$. 
Therefore,
showing
\begin{equation}
\Pr[\mG(t) \cap \mH(t)] = \Omega(1)
\end{equation}
suffices to show that regret-per-round is $\Omega(\myeps)$.
First, $X_{2,1}(t)$ is a Gaussian with scale $\Theta(1)$, which implies $\mG$ occurs with probability $\Theta(1)$. Second, conditioned on $\mG(t)$, the sufficient condition for $\mH(t)$ is
\begin{multline}
\mH'(t) := \{
\langle \Xarmt{1}, \hatthetaarm{1} \rangle < 0,
\langle \Xarmt{3}, \hatthetaarm{3} \rangle > 0, 
\langle \Xarmt{3}, \hatthetaarm{3} \rangle 
-
\langle \Xarmnt{2}{2,\setminus 1}, \hatthetaarm{2}_{\setminus 1} \rangle 
\ge 2 \myeps
\},
\end{multline}
where
\[
\langle \Xarmnt{2}{2,\setminus 1}, \hatthetaarm{2}_{\setminus 1} \rangle 
:= \langle \Xarmt{2}, \hatthetaarm{2} \rangle 
-
\Xarmnt{2}{1}\,\hatthetaarm{2}_{1},
\]
is an inner product that ignores the first component.
Here, under Eq.~\eqref{ineq_highvar}, each of $i \in [3]$, $\langle \Xarmt{i}, \hatthetaarm{i} \rangle$, and $\langle \Xarmnt{2}{2,\setminus 1}, \hatthetaarm{2}_{\setminus 1} \rangle$, is a Gaussian random variables with its standard deviation $\Omega(\myeps)$. By this fact it is clear that $\Pr[\mH'(t)] = \Omega(1)$.
Therefore, $\Pr[\mH'(t)|\mG(t)] = \Omega(1)$. In summary, $\mG(t) \cap \mH(t) \supset \mG(t) \cap \mH'(t)$ occurs with probability $\Omega(1)$, and thus regret-per-round is $\Omega(\myeps)$ by Eq.~\eqref{ineq_lower_regretperround}.

In summary, $\Ep[R_2] \ge \Omega(\myeps) \times T$,
and the expected regret is bounded as
\begin{align}
R(T) 
&= \Ep[R_1] + \Ep[R_2]\\
&= \Omega(T_0) + \Omega(\myeps) T,
\end{align}
which completes the proof.
\end{proof}

\section{Proofs for AEtC}

\subsection{Proof of Lemma \ref{lem_traceest}}

We have
\begin{align}\label{ineq_tracebound}
\lefteqn{
|\trace(\Sigmai) - \trace(\hatSigmai)| 
}\\
&= \log(T) \times O\left(\sqrt{ \sum_j \lambda_j^2 }\right)\\
&\text{\ \ \ (by Lemma \ref{lem_covest} with $\eta = O(\log(T))$, transformation above follows with probability $\poly(T^{-1})$)}\\
&= 
     \begin{cases}
         \tilde{O}\left(1\right)&\mbox{~(Example 1, where we used $\sum_j (j^{-(1+)})^2 = O(1)$),~} \\
         \tilde{O}\left(T^{c(1-2b)/2}\right)& \mbox{~(Example 2, where we used $\sum_{j=1}^{T^c} (j^{-b})^2 = [T^c]^{1-2b}$),~} 
     \end{cases}\\
&= o\left(\trace(\Sigmai)\right),\nn
&\text{\ \ \ (by $\trace(\Sigmai) = O(T^a)$ (Example 1) or $\trace(\Sigmai) = O(T^{c(1-b)})$ (Example 2))}
\end{align}
from which Lemma \ref{lem_traceest} follows. 

\begin{lemma}\label{lem_covest}
Suppose that Assumption \ref{asmp:subgaussian} holds.
Assume that $\lambda_1, \sum_j \lambda_j < + \infty$.
Let the empirical covariance matrix with $\Ni$ samples be 
\begin{align}
\hatSigmai := \frac{1}{\Ni} (\bfXi)^\top \bfXi.
\end{align}
Then, for any $\eta>1$, with probability at least $1 - 2 e^{-\eta}$, we have 
\[
|\trace(\hatSigmai) - \trace(\Sigmai)|
\le 
\frac{C \kappa_x^2 \max\left( \eta \lambda_1, \sqrt{\eta \sum_j \lambda_j^2}\right)}{\Ni}
\le 
C \kappa_x^2 \eta \sqrt{ \sum_j \lambda_j^2}
\]
for some constant $C>0$.
\end{lemma}

\begin{proof}[Proof of Lemma \ref{lem_covest}]
We have
\begin{align}
\trace(\hatSigmai) 
&= \frac{1}{\Ni} \trace\left(
(\bfXi)^\top \bfXi
\right)\\
&= \frac{1}{\Ni} \trace\left(
\Vi (\Lambdai)^{1/2} (\bfZi)^\top \bfZi 
(\Lambdai)^{1/2} (\Vi)^\top 
\right)\\
&= \frac{1}{\Ni} \trace\left(
(\Lambdai)^{1/2} (\Vi)^\top 
\Vi (\Lambdai)^{1/2} (\bfZi)^\top \bfZi 
\right)\\
&= \frac{1}{\Ni} \trace\left(
\Lambdai (\bfZi)^\top \bfZi 
\right),
\end{align}
and thus
\begin{equation}
\trace(\hatSigmai) - \trace(\Sigmai)
= \frac{1}{\Ni} \trace\left(
\Lambdai \left( (\bfZi)^\top \bfZi - I_p \right)
\right).
\end{equation}
Each diagonal element of $(\bfZi)^\top \bfZi - I_p $ is a sum of $\Ni$ independent samples, and each sample is zero-mean $\kappa_x^2$ sub-exponential random variable. Using Lemma 12 in \cite{bartlett2020benign}, with probability at least $1 - 2 e^{-\eta}$, we have
\begin{align}
\left|
\trace\left(
\Lambdai \left( (\bfZi)^\top \bfZi - I_p \right)
\right)
\right|
&\le C \kappa_x^2 \max\left( \eta \lambda_1, \sqrt{\eta \sum_j \lambda_j^2}\right),
\end{align}
for some $C>0$,
which completes the proof.
\end{proof}

\subsection{Proof of Lemma \ref{lem:eigenvalue}}

\begin{lemma}{\rm (Lemma 4.2 in \citet{bosq2000linear})}\label{lem_eigenbound}
For any two matrices $\bX_0,\bX_1$ with their eigenvalues $(\lambda_{0,j})_{j\in[p]}$ and $(\lambda_{1,j})_{j\in[p]}$, we have
\begin{equation}
|\lambda_{0,j} - \lambda_{1,j}| \le \|\bX_0 - \bX_1\|_{\mathrm{op}}\ \forall_{j\in[p]}.
\end{equation}
\end{lemma}
\begin{lemma}[Corollary 2 in \citet{koltchinskii2017concentration}]\label{lem_opbound}
Suppose that $X_1,...,X_N$ are $\R^p$-valued i.i.d. random vectors whose mean is zero and covariance is $\Sigma$.
$\hat{\Sigma} = N^{-1} \sum_{i=1}^N X_i X_i^\top$ is an empirical covariance matrix.
Then, there exists a constant $C>0$ and with probability $1-e^{-\eta}$ we have 
\begin{equation} 
\|\hat{\Sigma} - \Sigma\|_{\mathrm{op}}
\le C \|\Sigma\|_{\mathrm{op}}
\sqrt{\frac{\tilde{\br}(\Sigma) + \eta}{N}},
\end{equation}
where 
\[
\tilde{\br}(\Sigma) := \frac{\mathrm{tr}(\Sigma)}{\|\Sigma\|_{\mathrm{op}}}. 
\]
\end{lemma}

\begin{proof}[Proof of Lemma \ref{lem:eigenvalue}]
Lemmas \ref{lem_eigenbound} and \ref{lem_opbound} imply that 
\begin{align*}
    \max_j |\hat{\lambda}_j - \lambda_j| \leq \|\hat{\Sigma} - \Sigma\|_{\mathrm{op}} = 
    C \sqrt{\frac{\trace(\Sigmai) + \eta}{N}},
\end{align*}
with probability $1- e^{-\eta}$, and setting $\eta = \CPoly \log T$ yields the desired result.
\end{proof}

\subsection{Proof of Lemma \ref{lem_alphaest}}

This section adopts the same set of assumptions as Lemma \ref{lem_alphaest}.

\begin{lemma} \label{lem:betaest}
Let $\tau$ be a constant that is independent of $T$. Then, for any $\CPoly>0$, with probability at least $1-T^{-\CPoly}$, we have
\begin{equation}
|\beta_T - \hat{\beta}_T| = 
O\left(
\sqrt{\frac{\trace(\Sigmai) + \CPoly \log T}{N}}
\right).
\end{equation}
\end{lemma}
\begin{proof}[Proof of Lemma \ref{lem:betaest}]
We assume Eq.~\eqref{ineq_eigenerror} that holds with probability $1-T^{-\CPoly}$. We have,
\begin{align}
\beta_T - \hat{\beta}_T
&:=
\frac{1}{\tau}
\sum_{k=1}^\tau \frac{
\left(
\log(\lambda_k/\lambda_{k+1})
-
\log(\hat{\lambda}_k/\hat{\lambda}_{k+1})
\right)
}{
\log((k+1)/k)
}\\
&=
\frac{1}{\tau}
\sum_{k=1}^\tau \frac{
\left(
\beta_T\log(\frac{k+1}{k})
-
\log\left(\frac{(k+1)^{\beta_T}}{k^{\beta_T}} + 
C \sqrt{\frac{\trace(\Sigmai)+ \CPoly \log T}{N}}
\right)
\right)
}{
\log((k+1)/k)
}\\
&= O\left(\sqrt{\frac{\trace(\Sigmai)+ \CPoly \log T}{N}}\right),
\end{align}
where we have used 
$\log(c+x) = \log(c) + \log(1 + x/c) \approx \log(c) + \frac{x}{c} + o(x)$ in the last transformation.
\end{proof}

\begin{lemma} \label{lem:bound_small_r}
For $1 \leq k \leq N^{(i)}$ and any $\CPoly>0$, with probability at least $1-T^{-\CPoly}$, we have
    \begin{align*}
        \log\left(
        \frac{\hat{r}_k(\Sigmai)} {{r}_k(\Sigmai)}
        \right) = \log(T^{o(1)}).
    \end{align*}
    Moreover, we have 
    \begin{align*}
        \log\left(
            \frac{\hat{k}_n}{k_n^*}
        \right) = \log(T^{o(1)})
    \end{align*}
    for any $i \in[K], n \in [N^{(i)}]$.
\end{lemma}
\begin{proof}[Proof of Lemma \ref{lem:bound_small_r}]
We first show the first $k$ eigenvalues are negligible for $k = O(T)$: Namely,
\begin{align}\label{ineq:tailignore}
\frac{\sum_{j<k} \lambdaij }{\trace(\Sigmai)}
=
     \begin{cases}
         O\left(\frac{T^a (k^{1/T^a}-1)}{T^a}\right)&\mbox{~(Example 1),~} \\
         O\left(\frac{k^{1-b}}{T^{c(1-b)}}\right)& \mbox{~(Example 2),~} 
     \end{cases}
= o(1).
\end{align}

Moreover, we have
\begin{align}
    \log\left( \frac{\hat{r}_k(\Sigmai)} {{r}_k(\Sigmai)} \right) 
    &=\log\left(\left( \frac{ \trace(\hatSigmai) }{\sum_{j >k} {\lambda}_j^{(i)}} \right) \cdot  \left(\frac{{\lambda}_{k+1}^{(i)}}{\tilde{\lambda}_{k+1}^{(i)}} \right)  \right) \\
    &=\log\left(
    \left( \frac{ \trace(\hatSigmai) }{
    \trace(\Sigmai)
    } \right) \cdot  \left(\frac{{\lambda}_{k+1}^{(i)}}{\tilde{\lambda}_{k+1}^{(i)}} \right) \right) + o(1) 
    \text{\ \ \ \ (by Eq.~\eqref{ineq:tailignore})}
    \\
    &=\log\left(\left( \frac{ \trace(\hatSigmai) }{
    \trace(\Sigmai)
    } \right) \cdot  \left(\frac{(k+1)^{-\beta_T}}{(k+1)^{-\hat{\beta}_T}} \right) \right) + o(1)
    \\
    &=\log\left(\left( \frac{ \trace(\Sigmai) }{
    \trace(\Sigmai)
    }\right) \cdot  \left(\frac{(k+1)^{-\beta_T}}{(k+1)^{-\hat{\beta}_T}}\right)\right) + o(1)
    \text{\ \ \ \ (by Lemma \ref{lem_traceest})}
    \\
    &=\log\left(\left( \frac{ \trace(\Sigmai) }{
    \trace(\Sigmai)
    } \right) \cdot  \left(\frac{(k+1)^{-\beta_T}}{(k+1)^{-\beta_T}}\right)\right) + \log(T^{o(1)}) + o(1)
    \text{\ \ \ \ (by $k+1 = O(T)$ and Lemma \ref{lem:betaest})}
    \\
    &= \log(T^{o(1)}). \label{ineq_rdiff}
\end{align}
By definition,
\begin{align}
k^* &= \min \{k \geq 0 \mid r_k(\Sigmai) \geq N\}\\
\hat{k}_{N} &= \min \{k \geq 0 \mid  \hat{r}_k(\Sigmai) \geq N\}.
\end{align}
Eq.~\eqref{ineq_rdiff} states that $r_k(\Sigmai)/\hat{r}_k(\Sigmai) = T^{o(1)}$, and by using the fact that $\lambdaik = k^{-a}$ (Example 1) or $\lambdaik = k^{-b}$ (Example 2), we have $k^*/\hat{k}_{N} = T^{o(1)}$, which is equivalent to
\begin{equation}\label{ineq_logkdiff}
\log\left(\frac{k^*}{\hat{k}_{N}}\right) = \log(T^{o(1)}).
\end{equation}
\end{proof}

\begin{lemma} \label{lem:bound_large_R}
For any $n \in [N]$, with probability at least $1-T^{-\CPoly}$, we have
    \begin{align*}
        \log\left( \frac{ \hat{R}_{\hat{k}_{n}}(\Sigmai)}{ {R}_{{k}_{n}^*}(\Sigmai)} \right) = \log(T^{o(1)}).
    \end{align*}
\end{lemma}
\begin{proof}[Proof of Lemma \ref{lem:bound_large_R}]
As same to the proof of Lemma \ref{lem:bound_small_r}, we use Lemma \ref{lem:eigenvalue} and utilize $ |\hat{\lambda}_j^{(i)} - \lambda_j^{(i)} | \leq C \sqrt{\frac{\trace(\Sigmai) + \CPoly \log T}{N}}$ for every $j = 1,...,p$, which holds with probability $1-T^{-\CPoly}$.
Eq.~\eqref{ineq_logkdiff} implies that 
\[
\log\left(
\frac{
\hat{R}_{\hat{k}_n}(\Sigmai)
}{
\hat{R}_{{k}_n^*}(\Sigmai)
}
\right)
= T^{o(1)}.
\]
Then, we study the ratio as
\begin{align*}
     \log\left(\frac{ \hat{R}_{{k}_n^*}(\Sigmai)}{ {R}_{{k}_n^*}(\Sigmai)} \right) 
     &= \log\left( \left( \frac{\sum_{j > {k}_n^*} \lambdainum{j}}{\trace(\hatSigmai)} \right)^2 \frac{\sum_{j>\hat{k}_n} ( \tilde{\lambda}^{(i)}_j)^2}{\sum_{j > {k}_n^*} ({\lambda}_j^{(i)})^2} \right)\\
     &= \log\left( \left( \frac{\trace(\Sigmai)}{\trace(\hatSigmai)} \right)^2 \frac{\sum_{j>\hat{k}_n} ( \tilde{\lambda}^{(i)}_j)^2}{\sum_{j > {k}_n^*} ({\lambda}_j^{(i)})^2} \right) + o(1) 
     \text{\ \ \ \ (by Eq.~\eqref{ineq:tailignore})}\\
     &= \log\left( \left( \frac{\trace(\Sigmai)}{\trace(\Sigmai)} \right)^2 \frac{\sum_{j>\hat{k}_n} ( \tilde{\lambda}^{(i)}_j)^2}{\sum_{j > {k}_n^*} ({\lambda}_j^{(i)})^2} \right) + o(1)
    \text{\ \ \ \ (by Lemma \ref{lem_traceest})} \\
     &= \log\left( \left( \frac{\trace(\Sigmai)}{\trace(\Sigmai)} \right)^2 \frac{\sum_{j>\hat{k}_n} ( \lambda^{(i)}_j)^2}{\sum_{j > {k}_n^*} ({\lambda}_j^{(i)})^2} \right) + \log(T^{o(1)}) + o(1)
    \text{\ \ \ \ (by Lemma \ref{lem:betaest})} \\
     &= \log\left( \left( \frac{\trace(\Sigmai)}{\trace(\Sigmai)} \right)^2 \frac{\sum_{j>{k}_n^*} ( \lambda^{(i)}_j)^2}{\sum_{j > {k}_n^*} ({\lambda}_j^{(i)})^2} \right) + \log(T^{o(1)}) 
    \text{\ \ \ \ (by Eq.~\eqref{ineq_logkdiff})} \\
     &=  \log(T^{o(1)}).
\end{align*}
\end{proof}

\begin{proof}[Proof of Lemma \ref{lem_alphaest}]
The statement is straightforward from Lemmas \ref{lem:bound_small_r} and \ref{lem:bound_large_R}.
\end{proof}

\end{document}